\documentclass[11pt]{article}
\usepackage[letterpaper, margin=1in]{geometry}
\usepackage[utf8]{inputenc}

\usepackage{algorithm, algorithmic, authblk, enumerate, xcolor}
\usepackage[sort&compress,numbers]{natbib}

\newcommand{\N}{\mathbb{N}}
\newcommand{\R}{\mathbb{R}}

\newcommand{\keywords}[1]{\textbf{\textit{Keywords ---}} #1}

\usepackage{microtype}
\usepackage{graphicx}
\usepackage{subcaption}
\usepackage{booktabs} 

\newcommand{\LineComment}[1]{\hfill$\rhd\;$\textit{#1}}

\usepackage[hyphens]{url}
\usepackage{hyperref}

\usepackage{amsmath}
\usepackage{amssymb}
\usepackage{mathtools}
\usepackage{amsthm}

\usepackage[capitalize,noabbrev]{cleveref}

\theoremstyle{plain}
\newtheorem{theorem}{Theorem}[section]
\newtheorem{proposition}[theorem]{Proposition}
\newtheorem{lemma}[theorem]{Lemma}
\newtheorem{claim}[theorem]{Claim}

\theoremstyle{plain}
\newtheorem{definition}[theorem]{Definition}

\theoremstyle{remark}

\usepackage[textsize=tiny]{todonotes}

\begin{document}

\title{Smoothed Online Optimization with Unreliable Predictions}

\author[1]{Daan Rutten\thanks{Email: \href{mailto:drutten@gatech.edu}{drutten@gatech.edu}}\textsuperscript{,}}
\author[2]{Nicolas Christianson}
\author[1]{Debankur Mukherjee}
\author[2]{Adam Wierman}

\affil[1]{Georgia Institute of Technology}
\affil[2]{California Institute of Technology}

\maketitle
\keywords{online algorithms, non-convex optimization, competitive analysis}




\begin{abstract}
We examine the problem of smoothed online optimization, where a decision maker must sequentially choose points in a normed vector space to minimize the sum of per-round, non-convex hitting costs and the costs of switching decisions between rounds. The decision maker has access to a black-box oracle, such as a machine learning model, that provides untrusted and potentially inaccurate predictions of the optimal decision in each round. The goal of the decision maker is to exploit the predictions if they are accurate, while guaranteeing performance that is not much worse than the hindsight optimal sequence of decisions, even when predictions are inaccurate. We impose the standard assumption that hitting costs are globally $\alpha$-polyhedral. We propose a novel algorithm, Adaptive Online Switching (AOS), and prove that, for a large set of feasible $\delta > 0$, it is $(1+\delta)$-competitive if predictions are perfect, while also maintaining a uniformly bounded competitive ratio of $2^{\tilde{\mathcal{O}}(1/(\alpha \delta))}$ even when predictions are adversarial. Further, we prove that this trade-off is necessary and nearly optimal in the sense that \emph{any} deterministic algorithm which is $(1+\delta)$-competitive if predictions are perfect must be at least $2^{\tilde{\Omega}(1/(\alpha \delta))}$-competitive when predictions are inaccurate. In fact, we observe a unique threshold-type behavior in this trade-off: if $\delta$ is not in the set of feasible options, then \emph{no} algorithm is simultaneously $(1 + \delta)$-competitive if predictions are perfect and $\zeta$-competitive when predictions are inaccurate for any $\zeta < \infty$. Furthermore, we discuss that memory is crucial in AOS by proving that any algorithm that does not use memory cannot benefit from predictions. We complement our theoretical results by a numerical study on a microgrid application.
\end{abstract}

\section{Introduction}

We consider online (non-convex) optimization with switching costs in a normed vector space $(X, \lVert \cdot \rVert)$ wherein, at each time $t$, a decision-maker observes a non-convex \emph{hitting cost} function $f_t : X \to [0, \infty]$ and must decide upon some $x_t \in X$, paying $f_t(x_t) + \lVert x_t - x_{t-1} \rVert$, where $\lVert \cdot \rVert$ characterizes the \emph{switching cost}. Throughout, we assume that $f_t$ is globally $\alpha$-polyhedral (see Definition \ref{def:polyhedral}). Moreover, we assume that the decision maker has access to an \emph{untrusted prediction} $\tilde{x}_t$ of the optimal decision during each round, such as the decision suggested by a black-box AI tool for that round.

Online optimization is a problem with applications to many real-world problems \cite{lin2012dynamic, lu2012simple, kim2014real, cover1991universal, bansal2003online, wang2014exploring, joseph2012jointly, zanini2009multicore, zanini2010online}.
Usually there is a penalty for switching decisions too often: the goal of a decision maker is not just to minimize the hitting cost functions, but to also minimize the switching cost between rounds.
Online optimization with switching costs has received considerable attention in the learning, networking, and control communities in recent years \cite{lin2012online, goel2017thinking, goel2019online, li2018using, altschulerOnlineLearningFinite2018, shiCompetitiveOnlineConvex2021}. Moreover, the crux of many fundamental problems in online algorithms, such as the $k$-server problem \cite{koutsoupias1995k}, metrical task systems \cite{borodinOptimalOnlineAlgorithm1992, blum2000line}, and online convex body chasing \cite{sellkeChasingConvexBodies2020, argueChasingConvexBodies2021} is to handle switching costs.

The bulk of the literature on online optimization with switching costs has sought to design \emph{competitive} algorithms for the task, i.e., algorithms with finite competitive ratios. At this point, competitive algorithms for scenarios with both convex and non-convex hitting costs have been developed \cite{chen2018smoothed,lin2020online, zhang2021revisiting}.  However, while competitive analysis yields strong performance guarantees, it has often been criticized as being unduly pessimistic, since algorithms are characterized by their \emph{worst-case} performance, while worst-case conditions may never occur in practice.  As a result, an algorithm with an improved competitive ratio may not actually lead to better performance in realistic scenarios.

On the other hand, many real-world applications have access to vast amounts of historical data which could be leveraged by modern black-box AI tools in order to achieve significantly improved performance in the typical case. For example, in the context of capacity scaling for data centers, the historical data reveals reoccurring patterns in the weekly load of a data center.  AI models that are trained on these historical patterns can potentially outperform competitive algorithms in the typical case. This approach has been successfully used by Google in data center cooling \cite{deepmind2016}. 

Making use of modern black-box AI tools is potentially transformational for online optimization; however, such machine-learned algorithms fail to provide any uncertainty quantification and thus do not have formal guarantees on their worst-case performance. As such, while their performance may improve upon competitive algorithms in typical cases, they may perform arbitrarily worse in scenarios where the training examples are not representative of the real world workloads due to, e.g., distribution shift. This represents a significant drawback when considering the use of AI tools for safety-critical applications. 

A challenging open question is whether it is possible to provide guarantees that allow black-box AI tools to be used in safety-critical applications.  This paper aims to provide an algorithm that can achieve the best of both worlds -- making use of black-box AI tools to provide good performance in the typical case while integrating competitive algorithms to ensure formal worst-case guarantees.  We formalize this goal using the notions of \emph{consistency} and \emph{robustness} introduced by \cite{lykouris2018competitive} and used in the emerging literature on untrusted advice in online algorithms \cite{lykouris2018competitive, kumar2018semi, mitzenmacher2018model, hsu2019learning, rutten2021new, sun2021pareto, maghakian2021leveraging, lee2021online}. We consider that the online algorithm is given untrusted advice/predictions, e.g., the output of a black-box AI tool.  The predictions are fully arbitrary, i.e., we impose no statistical assumptions on the predictions, and the decision maker has no \emph{a priori} knowledge of the predictions' accuracy. The goal of the decision maker is to be both \emph{consistent} -- that is, to achieve performance comparable to that of the predictions if they are accurate, while remaining \emph{robust}, i.e., having cost that is never much worse than the hindsight optimal, even if predictions are completely inaccurate. Thus, an algorithm that is consistent and robust is able to match the performance of the black-box AI tool when the predictions are accurate while also ensuring a worst-case performance bound (something black-box AI tools cannot provide).  \\


\textbf{Main contributions.} 
We make five main contributions in this paper.
\emph{First}, 
we identify a fundamental trade-off between consistency and robustness for any deterministic algorithm. If an algorithm is $(1 + \delta)$-consistent, then this implies a lower bound on its robustness guarantee. In fact, we identify a region of $(\alpha, \delta)$, called the \emph{infeasible region} for which no algorithm can be simultaneously $(1 + \delta)$-consistent and $\zeta$-robust for any $\zeta < \infty$ (if the hitting costs are globally  $\alpha$-polyhedral). For any $(\alpha, \delta)$ outside the infeasible region, we prove that there is an exponential trade-off between consistency and robustness. More formally, Theorem \ref{thm:lowerbound} proves that any algorithm that is $(1 + \delta)$-consistent must be at least
\begin{equation}
    \frac{\alpha \delta}{4} \left( \frac{2}{\alpha + \delta (1 + \alpha)} \right)^{\frac{2 - \alpha (1 - \delta^2)}{\alpha \delta (1 + \delta)}} - \mathcal{O}(1) \quad \text{-robust}.
\end{equation}
This threshold-type behavior is unprecedented in the literature on learning-augmented algorithms and reveals the hardness of the problem instance brought about by the non-convexity of hitting costs.

\emph{Second,} we introduce a new algorithm for online non-convex optimization in a normed vector space with untrusted predictions, \emph{Adaptive Online Switching} (AOS), and provide bounds on its consistency and robustness. Our analysis shows that AOS can be used in combination with a black-box AI tool to match the performance of the black-box AI while also ensuring provable worst-case guarantees. 

The AOS algorithm works as follows. At each time $t$, AOS either follows the predictions $\tilde{x}_t$ or adopts a robust strategy that does not adapt to the predictions, and adaptively switches between these two. The challenge in the design of AOS is that, on the one hand, switching must be infrequent in order to limit the switching cost, but, on the other hand, switching must be frequent enough to ensure that the algorithm does not get stuck following a suboptimal sequence of decisions from either the predictions or the robust strategy. 

 Theorem \ref{thm:main} proves that, for $\alpha$-polyhedral hitting costs (see Definition \ref{def:polyhedral}), the competitive ratio $\textsc{CR}(\eta)$ of AOS is a function of the accuracy $\eta$ of the predictions and is at most
\begin{equation}
    \textsc{CR}(\eta) \leq (1 + \delta + \gamma) (1 + 2 \eta).
\end{equation}
Here, $\eta$ is an appropriate measure of the accuracy of the predictions (see Definition \ref{defn:accuracy}) and relates to the distance between the prediction $\tilde{x}_t$ and the optimal decision, and $\delta$ and $\gamma$ are hyperparameters of the algorithm. Note that, even though the competitive ratio is a function of the accuracy, the algorithm is oblivious to it. If the predictions are accurate, i.e., if $\eta = 0$, then the competitive ratio of AOS is $1 + \delta + \gamma$. In other words, AOS is $(1 + \delta + \gamma)$-consistent and almost reproduces the hindsight optimal sequence of decisions if the predictions are accurate. Moreover, if $(\alpha, \delta)$ belongs to the \emph{feasible region}, then
\begin{equation}
    \textsc{CR}(\eta) \leq \frac{12 + o(1)}{\gamma} \left( \frac{2}{\alpha + \delta (1 + \alpha)} \right)^{2 / (\alpha \delta)}.
\end{equation}
Therefore, even if predictions are completely inaccurate, i.e., if $\eta = \infty$, the competitive ratio of AOS is uniformly bounded.
The trade-off between consistency and robustness is characterized by the confidence hyperparameters $\delta$ and $\gamma$, where the robustness bound depends exponentially on both $\delta$ and $\alpha$. In light of our lower bound, this means that AOS reproduces the order optimal trade-off between robustness and consistency in the feasible region. As a proof technique, the conventional potential function approach fails due to the non-convexity of the problem and the incorporation of predictions. Hence, significant novelty in the technique is required for the proof of Theorem \ref{thm:main} (see Section \ref{sec:aos} for a discussion).

We complement the theoretical analysis of AOS's performance with empirical validation in Section \ref{sec:numerics}, where we report on experiments using AOS to ``robustify'' the decisions of a machine-learned algorithm for the problem of optimal microgrid dispatch with added noise and distribution shift in the predictions. These experiments confirm that AOS effectively bridges the good average-case performance of black-box AI with worst-case robust performance.

\emph{Third,} we extend the above results to the case when only an approximate non-convex solver is available. As we do not make any assumptions on the convexity of $f_t$, it may be computationally difficult or simply impossible to obtain the exact minimizer of $f_t$. We therefore extend AOS to work with any non-exact, approximate minimizer of $f_t$. AOS is oblivious to the accuracy of the solver and we prove that the competitive ratio decays smoothly in the accuracy of the solver. Moreover, we provide bounds for the case when predictions cannot be used in every time step due to the computational expense associated with the non-convex functions.  Our bounds characterize how the consistency-robustness tradeoff is impacted by this computational constraint. In fact, there the impact on the competitive ratio is linear in the number of time steps between available predictions.

\emph{Fourth,} we characterize the importance of memory for algorithms seeking to use untrusted predictions.  Interestingly, AOS requires full memory of all predictions.  This is a stark contrast with well-known memoryless algorithms for online optimization with switching costs, which do not make use of any information about previous hitting costs or actions. We prove in Theorem~\ref{thm:memoryless} that this contrast is no accident and that memory is needed in order to simultaneously achieve robustness and consistency. Thus, \emph{memory is necessary to benefit from untrusted predictions.}

\emph{Fifth,} we consider an important special case when the vector space is $X = \R$ and each function is convex and show that it is possible to provide an improved trade-off between robustness and consistency using a memoryless algorithm in this special case.  The one-dimensional online convex optimization setting has received significant attention in the literature, see \cite{bansal2015competitive, albersOptimalAlgorithmsRightSizing2018}. In this context, we introduce a new algorithm, called Adaptive Online Balanced Descent (AOBD), which is  inspired by Online Balanced Descent \cite{chen2015online}. AOBD has two tunable hyperparameters $\bar{\beta} \geq \underline{\beta} > 0$ that represent the confidence of the decision maker in the predictions.
Theorem \ref{thm:onedim} proves that the competitive ratio $\textsc{CR}(\eta)$ of AOBD is at most
\begin{equation}
    \textsc{CR}(\eta) \leq \min\{ \left( 1 + (2 + \bar{\beta}^{-1}) \underline{\beta} \right) (1 + 2 \eta), 1 + (2 + \underline{\beta}^{-1}) \bar{\beta} \}.
\end{equation}
The competitive ratio of AOBD has a similar structure to AOS, but improves the robustness bound significantly by taking advantage of the additional structure available compared to the general non-convex case. In particular, if $\bar{\beta} = 1 / \delta$ and $\underline{\beta} = \delta / (2 + \delta)$ for $\delta \leq 2$, then AOBD is $(1 + \delta)$-consistent and $1 + 3 / \delta + 2 / \delta^2$-robust. The result is complemented by a lower bound in Theorem \ref{thm:lbonedim} that proves that for $0<\delta<1/2$, any algorithm that is $(1 + \delta)$-consistent must be at least $1 / (2 \delta)$-robust. 

\section{Model Description}

We examine the problem of online non-convex optimization with switching costs on an arbitrary normed vector space $(X, \lVert \cdot \rVert)$. In this problem, the decision maker starts at an arbitrary point $x_0 \in X$ and is presented at each time $t = 1, \dots, T$ with a non-convex function $f_t : X \to [0, \infty]$. The decision maker must then choose an $x_t \in X$ and pay cost $f_t(x_t) + \lVert x_t - x_{t-1} \rVert$, i.e., the sum of the \emph{hitting cost} of $x_t$ and the \emph{switching cost} from $x_{t-1}$ to $x_t$. The goal of the decision maker is to solve
\begin{equation}\label{eq:prob-desc}
    \min_{x_t, 1 \leq t \leq T} \sum_{t = 1}^T \left( f_t(x_t) + \lVert x_t - x_{t-1} \rVert \right)
    \quad \text{s.t.} \quad x_t \in X \quad \text{ for } 1 \leq t \leq T.
\end{equation}
Note that the input is presented in an online fashion: at time $t$, the algorithm only knows $f_1, \dots, f_t$, but does not know $f_{t+1}, \dots, f_T$ or the finite time horizon $T$. We call a collection $(x_0, f_1, \dots, f_T, T)$ an \emph{instance} of the online non-convex optimization problem. The performance of the decision maker is compared to the hindsight optimal sequence of decisions, defined as
\begin{equation}
\label{eq:offline_opt}
    (o_1, \dots, o_T) \in \underset{y_1, \dots, y_T \in X}{\arg\min} \sum_{t=1}^T \left( f_t(y_t) + \lVert y_t - y_{t-1} \rVert \right),
\end{equation}
where we assume that the minimizer exists. 
We denote by $\textsc{Opt}(t) := f_t(o_t) + \lVert o_t - o_{t-1} \rVert$ the cost of the hindsight optimal algorithm at time $t$. 

We would like to emphasize the generality of our model. 
The decisions take values in an arbitrary normed vector space and the hitting cost functions can be non-convex, allowing our results to be applied to online decision-making in a variety of settings, including non-Euclidean and function spaces, as well as to problems where decisions are inherently discrete in nature such as right-sizing data centers~\cite{albersAlgorithmsRightSizingHeterogeneous2021}. Throughout, unless otherwise mentioned, we assume that the hitting cost functions are globally $\alpha$-polyhedral for some $\alpha > 0$, defined as follows.

\begin{definition}
\label{def:polyhedral}
We say that a function $f_t: X \to [0, \infty]$ is \textbf{globally $\alpha$-polyhedral} if it has a unique minimizer $v_t \in X$, and, for all $x \in X$,
$
    f_t(x) \geq f_t(v_t) + \alpha \cdot \lVert x - v_t \rVert.
$
\end{definition}

The assumption that hitting cost functions are $\alpha$-polyhedral is standard in the literature on online optimization with switching costs \cite{chen2018smoothed, lin2020online, zhang2021revisiting}. This type of structural assumption on hitting costs is in fact necessary to ensure the existence of a competitive algorithm for online optimization with switching costs on general vector spaces.\footnote{If hitting costs are allowed to be arbitrary, or if the only imposed assumption is that hitting costs are convex, then the problem class contains convex body chasing (CBC) as a special case \cite{sellkeChasingConvexBodies2020}. The competitive ratio of any algorithm for CBC on any $d$-dimensional normed vector space is $\Omega(\sqrt{d})$, so no algorithm can be competitive for CBC on an infinite-dimensional vector space.} The assumption that the switching cost is a metric is also crucial to our results; as we show in Appendix \ref{appendix:soco_squared_dist}, it is impossible to achieve simultaneous robustness and non-trivial consistency when the switching cost is a Bregman divergence, another common choice of switching cost, of which the squared Euclidean norm is a special case \cite{goel2019beyond}.

The fact that the hitting cost may be non-convex introduces new algorithmic challenges. For example, one cannot simply interpolate two points to obtain a convex combination of the hitting cost as done in state-of-the-art literature on online optimization without predictions \cite{bansal2015competitive, chen2018smoothed, goel2019beyond}. The non-convexity also poses computational challenges as the true minimizer of a non-convex optimization problem may not be attainable. Throughout, we therefore let our algorithms rely on approximate solvers, defined as follows.

\begin{definition}
Let $g: X \to [0, \infty]$ be any non-convex function. We say that $r$ \textbf{approximately solves} $g$, denoted by $r \approx \arg\min_{x \in X} g(x)$, if $r$ is the result of any non-convex solver applied to $g$. We define the \textbf{accuracy} of the solution $r$ as
\begin{equation}
    \varepsilon(r) := \frac{g(r)}{\min_{x \in X} g(x)} - 1.
\end{equation}
Also, if $r_1, r_2, \dots, r_T$ is the result of a sequence of non-convex solvers, we let $\varepsilon(r) := \max_{1 \leq t \leq T} \varepsilon(r_t)$.
\end{definition}

Although our performance guarantees naturally rely on the accuracy of the approximate solution, our algorithms do not require knowledge of the accuracy upfront.


\subsection{Untrusted Predictions} We assume that the decision maker has access to potentially inaccurate predictions of the hindsight optimal sequence of decisions. At time $t$, before choosing an action, the decision maker receives a suggested action $\tilde{x}_t \in X$.
We want to emphasize that these predictions are of the \emph{optimal decisions}, and not of the hitting cost functions, as has been studied previously \cite{chen2015online, chen2016using, li2018using, li2020leveraging, lin2020online}. This new model captures scenarios where, for example, a black-box machine learning algorithm is available and outputs $\tilde{x}_t$ as the suggested action for time $t$. Alternatively, $\tilde{x}_t$ could be the result of using imperfect forecasts of future hitting cost functions via a look-ahead algorithm for online optimization such as SFHC \cite{lin2020online}.

The accuracy of the predictions is measured in terms of the distance to the hindsight optimal sequence of decisions, which we quantify as follows. 
\begin{definition}
\label{defn:accuracy}
Consider an instance $(x_0, f_1, \dots, f_T, T)$ for the online non-convex optimization problem and let $\tilde{x} = (\tilde{x}_1, \dots, \tilde{x}_T) \in X^T$ be a prediction sequence. We say that the prediction $\tilde{x}$ is \textbf{$\eta$-accurate} for the instance if
\begin{equation}
\label{eq:eta-accurate}
    \sum_{t = 1}^T \lVert o_t - \tilde{x}_t \rVert \leq \eta \sum_{t = 1}^T \textsc{Opt}(t),
\end{equation}
where $o_t$ is the optimal sequence of decisions as defined in \eqref{eq:offline_opt}.
\end{definition}

Note that the accuracy is normalized in terms of the total cost of the hindsight optimal algorithm to make the notion scale-invariant. 
For example, if the norm $\lVert \cdot \rVert$ is doubled, then both the left-hand side of \eqref{eq:eta-accurate} and the optimal cost double, and hence, $\eta$ stays the same. This is natural since the quality of the predictions in this case does not change.

\subsection{Defining Consistency and Robustness} We measure the performance of an algorithm using the competitive ratio, which is a function of the accuracy in our setting.


\begin{definition} \label{defn:competitive_ratio}
Let $\mathcal{A}$ be an algorithm for the online non-convex optimization problem in \eqref{eq:prob-desc} that adapts to the predictions. The \textbf{competitive ratio} of $\mathcal{A}$ is $\textsc{CR}(\eta)$, or $\mathcal{A}$ is $\textsc{CR}(\eta)$\textbf{-competitive}, if
\begin{equation} \label{eq:competitive_ratio_defn}
    \sum_{i = 1}^T \left( f_t(x_t) + \lVert x_t - x_{t-1} \rVert \right) \leq \textsc{CR}(\eta) \cdot \sum_{t = 1}^T \textsc{Opt}(t),
\end{equation}
for all instances $(x_0, f_1, \dots, f_T, T)$ and all $\eta$-accurate predictions $\tilde{x}$. 

\end{definition}

The aim of this work is to design an algorithm for which the competitive ratio improves with the quality of the predictions. We quantify this in terms of the notions of \emph{consistency} and \emph{robustness}, which have recently emerged as important measures for the ability of algorithms to effectively make use of untrusted predictions in online settings \cite{lykouris2018competitive, kumar2018semi, mitzenmacher2018model, hsu2019learning, rutten2021new, sun2021pareto, maghakian2021leveraging, lee2021online}.

\begin{definition}
Let $\mathcal{A}$ be an algorithm for the online non-convex optimization problem in \eqref{eq:prob-desc} and let $\textsc{CR}(\eta)$ be its competitive ratio when it has access to $\eta$-accurate predictions. We say that $\mathcal{A}$ is
\\
{\normalfont (a)} \textbf{$(1 + \delta)$-consistent} if $\textsc{CR}(0) \leq 1 + \delta$; $\quad$
{\normalfont (b)} \textbf{$\zeta$-robust} if $\textsc{CR}(\eta) \leq \zeta$ for any $\eta \in [0, \infty]$.
\end{definition}

An algorithm with these qualities of consistency and robustness is guaranteed to have near-optimal performance when predictions are perfect, along with a constant competitive ratio even when predictions are arbitrarily bad. If predictions are, for example, the decisions made by a black-box AI algorithm, a robust and consistent algorithm ensures a worst-case performance guarantee without having to sacrifice the AI algorithm's typically excellent performance. Moreover, the algorithm we propose in this work has performance that \emph{degrades gracefully in $\eta$}: that is, even if predictions are not exactly perfect but are near-optimal, our algorithm will still achieve near-optimal cost.

\subsection{Warmup: Achieving Consistency Without Robustness}

By definition, following the predictions exactly guarantees 1-consistency. However, if the hitting cost functions are steep and predictions are not perfect, naively following the predictions does not yield a competitive ratio with a smooth dependence on the accuracy $\eta$. By suitably ``filtering'' the predictions, it is possible to achieve a competitive ratio that is linear in $\eta$. 

\begin{algorithm}[t]
\begin{algorithmic}
\STATE $p_0 \leftarrow x_0$
\FOR{$t = 1, \ldots, T$}
    \STATE $p_t \approx \underset{p \in X}{\arg\min}\, f_t(p) + 2 \lVert p - \tilde{x}_t \rVert$ \label{eq:ftp}
\ENDFOR
\end{algorithmic}
\caption{Follow the Prediction (FtP)}
\label{alg:ftp}
\end{algorithm}

To this end, we present in Algorithm \ref{alg:ftp} the filtering procedure called \emph{Follow the Prediction} ($\textsc{FtP}$) that, at each time~$t$, moves to a point $p_t$ constituting a ``filtered'' form of the prediction $\tilde{x}_t$. The $\textsc{FtP}$ algorithm is the same algorithm in as in \cite{antoniadisOnlineMetricAlgorithms2020}, with an adjustment so that it uses an approximate solver rather than an exact solver, since an exact solver may not be available for non-convex problems. As we shall see in the following lemma, the competitive ratio of \textsc{FtP} is linear in $\eta$.

\begin{lemma}
\label{lemma:ftpalg}
Let $\textsc{CR}(\eta)$ be the competitive ratio of \textsc{FtP}. Then,
\begin{equation} \label{eq:ftp_cr_bound}
    \textsc{CR}(\eta) \leq 1 + \varepsilon(p) + \left( 4 + 2 \varepsilon(p) \right) \eta.
\end{equation}
\end{lemma}

Note that Lemma \ref{lemma:ftpalg} does not guarantee that $\textsc{FtP}$ is worst-case competitive, a.k.a.\ robust, and in fact, it is relatively straightforward to construct example settings where $\eta$ is arbitrarily large and $\textsc{FtP}$ has an unbounded competitive ratio.

\section{Main Results}


Although consistency and robustness are nontrivial to achieve simultaneously, it is straightforward to obtain consistency and robustness individually. On the one hand, we saw in the previous section that FtP guarantees consistency, but not robustness. On the other hand, an algorithm that follows the minimizers $v_t$ is $\max\{2 / \alpha, 1\}$-robust \cite{zhang2021revisiting}, but not consistent. Hence, our goal in this section is to combine the features of both algorithms to construct an algorithm that is both consistent and robust.

A natural algorithm to consider is a `switching type' algorithm, which chooses between following the minimizers and the predictions. These switching-type algorithms have already been used successfully in general problem settings \cite{antoniadisOnlineMetricAlgorithms2020}. In fact, due to the non-convexity of the hitting cost functions, any reasonable algorithm must be a switching-type algorithm, since the cost of any state in between the minimizer $v_t$ and the prediction $\tilde{x}_t$ is generally unrelated to and potentially much higher than the hitting cost at $v_t$ or $\tilde{x}_t$ itself. The difficulty of designing these algorithms in our case lies in the notion of consistency, which dictates that the cost of the algorithm must only be a small fraction higher than the cost of the predictions. This requires a careful design of the algorithm, which should only switch away from the prediction if it holds a strong conviction that robustness would otherwise be violated.

\begin{algorithm}[ht]
\begin{algorithmic}[1]
\STATE Let $r_t \approx \arg\min_{r \in X} f_t(r)$, $\textsc{Rob}(t) := f_t(r_t) + \lVert r_t - r_{t-1} \rVert$, and \\ \noindent $\textsc{Adv}(t) := f_t(p_t) + \lVert p_t - p_{t-1} \rVert$ for $t = 1, 2, \dots, T$.
\STATE $T_1 \leftarrow 1, \; t \leftarrow 1$\;
\FOR{$k = 1, 2, \dots$}
    \WHILE{$\sum_{i = T_k}^{t-1} \textsc{Adv}(i) + \textsc{Rob}(t) + \lVert p_{t-1} - r_{t-1} \rVert + \lVert r_t - p_t \rVert \geq (1 + \delta) \sum_{i = T_k}^t \textsc{Adv}(i)$} \label{line:whileadv}
        \STATE $x_t \leftarrow p_t, \; t \leftarrow t + 1$ and stop if $t = T + 1$
    \ENDWHILE
    \STATE $M_k \leftarrow t$ \LineComment{Start following the minimizers}
    \STATE $x_t \leftarrow r_t, \; t \leftarrow t + 1$ and stop if $t = T + 1$
    \WHILE{$\sum_{i = M_k+1}^t \textsc{Rob}(i) {+} \lVert r_t - p_t \rVert {-} \lVert r_{M_k} - p_{M_k} \rVert \leq (1 + \delta) \sum_{i = M_k+1}^t \textsc{Adv}(i) + \gamma \sum_{i = T_k}^t \textsc{Adv}(i)$} \label{line:whilerob}
        \STATE $x_t \leftarrow r_t, \; t \leftarrow t + 1$ and stop if $t = T + 1$
    \ENDWHILE
    \STATE $T_{k+1} \leftarrow t$ \LineComment{Start following the predictions}
    \STATE $x_t \leftarrow p_t, \; t \leftarrow t + 1$ and stop if $t = T + 1$
\ENDFOR
\end{algorithmic}
\caption{Adaptive Online Switching (AOS)}
\label{alg:main}
\end{algorithm}

We propose Algorithm~\ref{alg:main}, named \emph{Adaptive Online Switching} (AOS), which adaptively switches between following the approximate minimizers and filtered versions of the predictions in a manner that ensures both robustness and consistency. The conditions in lines \ref{line:whileadv} and \ref{line:whilerob} dictate when the algorithm should switch and are carefully designed to simultaneously guarantee consistency and robustness. The AOS algorithm has hyperparameters $\delta, \gamma>0$, which represents the confidence in the predictions. The smaller $\delta$ and $\gamma$, the closer the performance to the predictions, while sacrificing robustness in the worst-case.

\subsection{Bounding Consistency and Robustness}
\label{sec:aos}

We prove that AOS is always consistent. However, to our surprise, AOS is only robust for specific choices of $\delta$, which are dependent on the value of $\alpha$. If $\delta$ is too small, or in other words, if AOS trusts the predictions too much, then AOS is not robust. There is no such restriction the parameter $\gamma$. The $(\alpha, \delta)$ region for which AOS is robust is called the feasible region, as defined below.

\begin{definition}
Fix any $(\alpha, \delta)$ and let $U(t)$ be given by
\begin{equation}
\begin{aligned}
\label{eq:duallp}
   U(t) := &\min_{U, y} \quad U \\
   \text{s.t.} \quad \delta \sum_{i = s}^t (1 + \alpha (i - s + 1)) y_i + \alpha \sum_{i = s}^{t-1} y_i &- 2 \sum_{i = s+1}^t y_i \geq 1 + \alpha (t - s + 1) &\text{ for } 1 \leq s \leq t, \\
   U \geq 2 y_s + \delta \sum_{i = s}^t y_i &- 1, \;
   y_s \geq 0 &\text{ for } 1 \leq s \leq t.
\end{aligned}
\end{equation}
We say that $(\alpha, \delta)$ is \textbf{feasible} if $\sup_{t \geq 1} U(t) < \infty$.
\end{definition}

\begin{figure}
    \centering
    \includegraphics[width=0.6\textwidth]{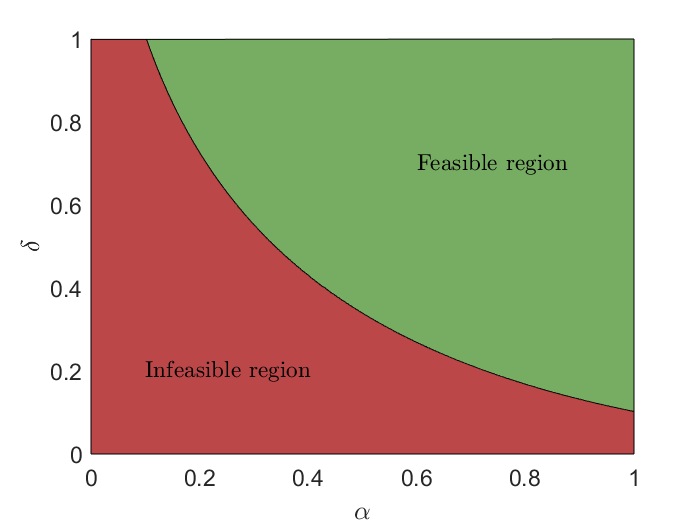}
    \caption{The feasible and infeasible region for $0 \leq \alpha \leq 1$ and $0 \leq \delta \leq 1$.}
    \label{fig:feasibleregion}
\end{figure}

Although a closed-form characterization of the feasible region is not available, numerically computing the feasible region is quite fast. Figure \ref{fig:feasibleregion} shows the feasible region for $0 \leq \alpha \leq 1$ and $0 \leq \delta \leq 1$. For example, $\delta = 0.5$ is feasible for $\alpha = 0.5$ but not for $\alpha = 0.2$. The value for which consistency is achievable therefore critically relies on $\alpha$. The next theorem shows that AOS is both consistent and robust in this region and the parameters $\delta$ and $\gamma$ conveniently control the trade-off between performance in the typical case when $\eta$ is small and the worst-case when $\eta$ is large, in hindsight.

\begin{theorem}
\label{thm:main}
Let $\textsc{CR}(\eta)$ be the competitive ratio of AOS. Then,
\begin{equation}
    \textsc{CR}(\eta) \leq (1 + \delta + \gamma) (1 + \varepsilon(p) +  (4 + 2 \varepsilon(p)) \eta).
\end{equation}
Moreover, if $(\alpha, \delta)$ is feasible and $\frac{2}{\alpha \delta} \in \N$. Then,
\begin{equation}
    \textsc{CR}(\eta) \leq \frac{12 + o(1)}{\gamma} \left( 1 + \varepsilon(r) + \frac{2 \varepsilon(r)}{\alpha} \right) \left( \frac{2}{\alpha + \delta (1 + \alpha)} \right)^{2 / (\alpha \delta)}.
\end{equation}
\end{theorem}

The assumption that $2 / (\alpha \delta) \in \N$ is without loss of generality and prevents rounding symbols from appearing in the notation. The asymptotic notation $o(1)$ refers to the asymptotic regime $\alpha, \delta, \gamma \to 0$. An exact, non-asymptotic characterization of the competitive ratio is available in the appendix.
The competitive ratio of AOS is the minimum of two terms. If the predictions and the non-convex solver are accurate, i.e., $\eta$ and $\varepsilon(p)$ are close to $0$, then the competitive ratio will be close to $1 + \delta + \gamma$. In particular, this means that the AOS algorithm is $(1 + \delta + \gamma)$-consistent if $\varepsilon(p) = 0$. Moreover, even for inaccurate predictions, i.e., as $\eta \to \infty$, the AOS algorithm has a bounded competitive ratio, implying robustness. These simultaneous properties of AOS enable its use to filter the decisions of untrusted, black-box AI algorithms, yielding worst-case performance guarantees without giving up good average-case performance in safety-critical settings.


The confidence hyperparameters $\delta$ and $\gamma$ describe how much faith the algorithm designer has in the predictions and consequently influences the performance guarantees. 
As mentioned before, $\delta$ and $\gamma$ describe the trade-off between the performance in the cases when predictions are accurate and completely inaccurate, respectively. Theorem \ref{thm:main} explicitly describes the trade-off between these two scenarios and interpolates to any scenario in between. 

It is worthwhile to highlight that the proof of Theorem \ref{thm:main} does not follow the conventional potential function approach, since this approach is hindered by its dependence on both the optimal trajectory and the trajectory of the predictions (as opposed to only the optimal trajectory in conventional competitive analysis). Although the consistency bound is not hard to prove, the proof of the robustness bound is quite involved and requires significant novelty. For example, even the fact that the AOS algorithm eventually switches to the minimizers and does not follow the advice forever, is not obvious. The reason that the AOS algorithm eventually switches to the robust algorithm and is robust itself critically depends on the globally $\alpha$-polyhedral assumption. Intuitively, to see why the AOS algorithm will switch eventually, consider the following example. Assume that $r_t = v_t$, the minimizer $v_t$ is static and the advice moves a distance of 1 away from the minimizer in each round. Then, the accumulated switching cost is equal to $t$ at time $t$. However, the accumulated hitting cost is at least $\alpha \sum_{i = 1}^t i = \alpha t (t + 1) / 2$ at time $t$. Hence, the accumulated hitting cost increases \emph{quadratically}, while the accumulated switching cost increases \emph{linearly}. In general, the accumulated hitting cost increases \emph{at a faster rate} than the accumulated switching cost. As the distance to the minimizer is upper bounded by the accumulated switching cost, this means that there must be a $t$ for which
\begin{equation}
    \lVert p_{t-1} - r_{t-1} \rVert + \lVert r_t - p_t \rVert
    = t-1 + t
    < \frac{\alpha \delta (t-1) t}{2}
    \leq \delta \sum_{i = 1}^{t-1} \textsc{Adv}(t),
\end{equation}
which is sufficient to invalidate the condition in line \ref{line:whileadv}, and hence, the AOS algorithm will switch to the robust algorithm. In general, not only the time until a switch is bounded, one can even bound the cost of the advice until this time is reached by viewing the problem as a non-convex optimization problem, where the goal of the adversary is to maximize the cost of the AOS algorithm while minimizing the cost of the optimal trajectory simultaneously. We prove that there exists a simpler linear optimization problem that forms a tight bound on this non-convex problem, which leads to the definition of $U(t)$ in \eqref{eq:duallp}. The proof of Theorem \ref{thm:main} can be found in Appendix \ref{app:mainproof}.

\subsection{A Lower Bound}

To understand whether it is possible to improve on the performance of AOS, we turn our attention to a lower bound. We show in this section that there is a fundamental trade-off between consistency and robustness for \emph{any} algorithm and that this trade-off is exponential. Moreover, if $\alpha$ and $\delta$ are infeasible then \emph{no} algorithm can be simultaneously $(1 + \delta)$-consistent and $\zeta$-robust for $\zeta < \infty$. We define the infeasibility region as follows.

\begin{definition}
Fix any $(\alpha, \delta)$ and let $L(t)$ be given by
\begin{equation}
\begin{aligned}
\label{eq:lboptproblem}
    L(t) := \max_{\Delta} \quad &\sum_{i = 1}^t \left( \Delta_i + \alpha \left( 1 + \sum_{j = 1}^i \Delta_j \right) \right) \\
    \text{s.t.} \quad 2 \left(1 + \sum_{i = 1}^{s-1} \Delta_i \right) &\geq \delta \sum_{i = 1}^s \left( \Delta_i + \alpha \left( 1 + \sum_{j = 1}^i \Delta_j \right) \right) + \alpha \left( 1 + \sum_{i = 1}^s \Delta_i \right) &\text{ for } 1 \leq s \leq t, \\
    \Delta_s &\geq 0 &\text{ for } 1 \leq s \leq t.
\end{aligned}
\end{equation}
We say that $(\alpha, \delta)$ is \textbf{infeasible} if $\sup_{t \geq 1} L(t) = \infty$.
\end{definition}

Note that we have a closed-form solution for neither the feasible nor the infeasible region. As a result, it is not possible to verify analytically whether each $(\alpha, \delta)$ is feasible if it is not infeasible. We therefore compute and compare the feasible and infeasible region numerically, which is shown in Figure \ref{fig:feasibleregion}. The feasible and infeasible region leave only a very small gap, which we verify numerically is less than $10^{-3}$. The next theorem shows that no algorithm can be robust inside the infeasible region and any algorithm must have an exponential trade-off outside the infeasible region.

\begin{theorem}
\label{thm:lowerbound}
Fix any $\delta > 0$, and assume that $(2 - \alpha (1 - \delta^2)) / (\alpha \delta (1 + \delta)) \in \N$. Let $\mathcal{A}$ be any deterministic online algorithm for the non-convex optimization problem in \eqref{eq:prob-desc}. If there exists $\varepsilon < \delta$ such that $\mathcal{A}$ is $(1 + \varepsilon)$-consistent, then $\mathcal{A}$ is at least $L$ robust, where
\begin{equation}
    L = \begin{cases}
        \frac{\alpha \delta}{4} \left( \frac{2}{\alpha + \delta (1 + \alpha)} \right)^{\frac{2 - \alpha (1 - \delta^2)}{\alpha \delta (1 + \delta)}} - \mathcal{O}(1) &\text{ if } (\alpha, \delta) \text{ is not infeasible,}\\
        \infty &\text{ if } (\alpha, \delta) \text{ is infeasible.}
    \end{cases}
\end{equation}
\end{theorem}

Note that the assumption that $(2 - \alpha (1 - \delta^2)) / (\alpha \delta (1 + \delta)) \in \N$ is again without loss of generality, but prevents rounding symbols from appearing in the notation.
Theorem \ref{thm:lowerbound} highlights that any deterministic online algorithm must experience a trade-off between consistency and robustness. In particular, no deterministic algorithm can be $1$-consistent and robust. So far, it has been rare in the literature to have provably optimal guarantees for the trade-off between robustness and consistency \cite{rutten2021new, sun2021pareto}, and even more unique if the bound is actually achievable. Moreover, to the best of our knowledge, we are the first to identify the threshold type behavior where consistency and robustness are \emph{not} simultaneously achievable for specific choices of $\alpha$ and $\delta$. The proof of Theorem \ref{thm:lowerbound} is provided in Appendix \ref{app:proof-lower}. 

\subsection{Using Predictions Frugally}
In some applications, it may not be feasible to obtain predictions at every timestep, due to computational cost, complexity of running the machine-learned model at every timestep or the nonconvexity of the problem setting.  This motivates the development of algorithms that can make \emph{frugal} use of predictions, i.e., which do not utilize predictions at every timestep, yet which still obtain bounds on robustness and consistency. We consider a setup in which a decision-maker wishes to use predictions only every $k$ timesteps, for some $k \in \N$; as we will see, for Lipschitz continuous hitting costs, this merely degrades the consistency bound by a multiple of $k$ multiplied by the Lipschitz constant. In particular, the following proposition gives a competitive bound dependent on the extent of frugality for any algorithm.

\begin{proposition} \label{prop:frugal_use}
    Let $\mathcal{A}$ be any algorithm with competitive ratio $\textsc{CR}$. Assume that each hitting cost $f_t$ is $L$-Lipschitz continuous. If $\mathcal{A}$ produces decisions $x_0, x_1, \ldots, x_T$, then the algorithm $\mathcal{A}'$ that, at time $t$, produces the decision $x_{k \cdot \lfloor t/k \rfloor}$, has competitive ratio at most $\max\{1, kL\}\cdot\textsc{CR}$.
\end{proposition}
The proof of the proposition is given in Appendix~\ref{sec:prooffrugal}. Observe that the algorithm $\mathcal{A}'$ in the above proposition acts by repeating for $k$ timesteps the decisions made by $\mathcal{A}$ at times that are multiples of $k$. Thus, $\mathcal{A}'$ only needs to query $\mathcal{A}$ once every $k$ timesteps, yet still maintains a competitive ratio close to that of $\mathcal{A}$. Moreover, when the Lipschitz constant $L$ of the hitting cost functions is small, $\mathcal{A}'$ obtains performance that is very close to that of $\mathcal{A}$; in particular, when $L = k^{-1}$, their performance is identical.

It is straightforward to observe that Proposition~\ref{prop:frugal_use} enables obtaining robustness and consistency bounds when predictions are used only frugally. Applying the proposition to the predictions that are used in AOS, we obtain an analog of Theorem~\ref{thm:main} with an extra multiple of $\max\{1, kL\}$ in the consistency bound; thus we can obtain both robustness and consistency with respect to the full sequence of predictions while only utilizing the predictions every $k$ timesteps. Moreover, this same reasoning can be applied to the decisions made by AOS if both the predictions and the minimizers $r_t$ can only be used frugally, due to computational challenges associated with the nonconvexity of the problem.  In this case, the result would be identical to Theorem~\ref{thm:main} with extra multiples of $\max\{1, kL\}$ in the consistency and robustness bounds.

\subsection{Limitations of Memoryless Algorithms \label{subsection:memoryless}}
Algorithm \ref{alg:main} requires the accumulated cost of following the minimizers and the predictions until the current time to be stored in memory, in contrast to most previous algorithms for online optimization. These algorithms have typically been memoryless, i.e., the decision at time $t$ is only based on $x_{t-1}$ and the current hitting cost function $f_t$ \cite{chen2018smoothed, goel2019beyond, zhang2021revisiting}. In light of this discrepancy, it is natural to ask the following question: \textit{can there exist a memoryless algorithm that is $(1+\delta)$-consistent and robust?}

We answer this question in the negative in Theorem \ref{thm:memoryless}. We show that, under mild assumptions of scale- and rotation-invariance, any memoryless algorithm achieving finite robustness must have consistency of $\Omega\left(\alpha^{-1/2}\right)$.

\begin{definition}
    An algorithm $\mathcal{A}$ for online optimization with switching costs and predictions is \textbf{memoryless} if its decision $x_t$ at time $t$ depends only on $x_{t-1}$, $\tilde{x}_t$, and $f_t$.
\end{definition}

\begin{definition}
    A memoryless algorithm $\mathcal{A}$ for online optimization with switching costs and predictions is \textbf{scale-invariant} if the behavior of $\mathcal{A}$ is invariant under scaling of $X$. That is, if $x_t$ is the decision $\mathcal{A}$ makes at time $t$ when faced with $x_{t-1}$, $\tilde{x}_t$, and $f_t$, then $\mathcal{A}$ is scale-invariant if, for all $\lambda > 0$, its decision is $\lambda x_t$ when faced with previous decision $\lambda x_{t-1}$, prediction $\lambda \tilde{x}_t$, and hitting cost $f_t(\lambda\cdot)$. 
    
    We say $\mathcal{A}$ is \textbf{rotation-invariant} if, when $X$ is a real inner product space, the behavior of $\mathcal{A}$ is invariant under rotation of $X$. That is, if for all rotation operators $U$, $\mathcal{A}$ makes decision $Ux_t$ when its previous decision was $Ux_{t-1}$, advice is $Up_t$, and the hitting cost is $f_t(U\cdot)$.
\end{definition}

\begin{theorem}
\label{thm:memoryless}
Let $\mathcal{A}$ be any memoryless and scale- and rotation-invariant algorithm, and let $\alpha < 1 / 4$. The next two statements are mutually exclusive: (i) $\mathcal{A}$ is $\zeta$-robust for $\zeta < \infty$ and (ii) $\mathcal{A}$ is $c$-consistent with 
    $
        c < 1/\sqrt{8 \alpha} - o\left(1/\sqrt{\alpha} \right).
    $
\end{theorem}

A proof of Theorem \ref{thm:memoryless} is provided in Appendix \ref{sec:proofmemoryless}.
Theorem \ref{thm:memoryless} shows that there is a fundamental lower bound of $\Omega\left(\alpha^{-1/2}\right)$ on the consistency of any robust, memoryless algorithm. In particular, as $\alpha \to 0$, this lower bound grows arbitrarily large and hence simultaneous robustness and near-perfect consistency cannot be expected from any memoryless algorithm. Thus, memoryless algorithms cannot match the performance of AOS.

Notably, the lower bound holds even for online \emph{convex} optimization and matches the best known upper bound of $\mathcal{O}\left(\alpha^{-1/2}\right)$ on the competitive ratio of algorithms for online convex optimization without predictions \cite{lin_soco_report}. In conjunction with the lower bound in Theorem \ref{thm:memoryless}, this implies that, within the restricted class of memoryless, scale- and rotation-invariant algorithms, optimal robustness and consistency can be achieved while ignoring predictions. In other words, memory is a fundamental requirement for any algorithm to benefit from unreliable predictions.

\subsection{A Memoryless Algorithm for Online Convex Optimization in One Dimension}
\label{sec:convex1d}

Although the previous section's lower bound gave a pessimistic view of the prospect of designing memoryless algorithms that achieve robustness and consistency, it leaves open the question of whether it is possible to do better when there are stronger assumptions on the problem. We answer this question in the affirmative for the following special case: the normed vector space space is $X = \R$ and $\lVert x - y \rVert = \lvert x - y \rvert$ is the Euclidean distance; each function $f_t: \R \to [0, \infty]$ is convex.
We assume the above two conditions throughout this subsection.
Note that we do \emph{not} assume that $f_t$ is globally $\alpha$-polyhedral, as opposed to the previous section and the general problem description.

The one-dimensional case has been considered in depth in the literature as an important problem in its own right, and specialized algorithms for this case are known \cite{bansal2015competitive, antoniadis2017tight, lin2012dynamic}. 
We use the Online Balanced Descent algorithm \cite{chen2018smoothed} as inspiration for the robust algorithm.\footnote{Note that extensions of OBD with improved competitive ratios are known (for example G-OBD and R-OBD \cite{goel2019beyond}), but these do not work well in the case where the norm is the Euclidean distance.}
However, since each $f_t$ is now convex, unlike the AOS algorithm, the algorithm does not need to follow either the predictions or the robust algorithm exactly.
Instead, it smoothly interpolates between the predictions and the robust algorithm.

\begin{algorithm}[t]
\begin{algorithmic}[1]
\FOR{$t = 1, 2, \dots, T$}
    \STATE Observe $f_t$ and $\tilde{x}_t$
    \STATE Let $x(\lambda) := (1 - \lambda) x_{t-1} + \lambda v_t$ for $\lambda \in [0, 1]$
    \STATE Let $\underline{\lambda} \in [0, 1]$ be such that $\lvert x(\underline{\lambda}) - x_{t-1} \rvert = \underline{\beta} f_t(x(\underline{\lambda}))$ or $\underline{\lambda} = 1$ if such a $\underline{\lambda}$ does not exist \label{line:optpr1}
    \STATE Let $\bar{\lambda} \in [0, 1]$ be such that $\lvert x(\bar{\lambda}) - x_{t-1} \rvert = \bar{\beta} f_t(x(\bar{\lambda}))$ or $\bar{\lambda} = 1$ if such a $\bar{\lambda}$ does not exist \label{line:optpr2}
    \STATE $p_t \leftarrow \underset{p \in X}{\arg\min}\, f_t(p) + \lVert p - p_{t-1} \rVert + \lVert p - \tilde{x}_t \rVert$
    \STATE $\lambda \leftarrow \arg\min_{\lambda \in [\underline{\lambda}, \bar{\lambda}]} \lvert x(\lambda) - p_t \rvert$ \label{line:chooselambda}
    \STATE $x_t \leftarrow x(\lambda)$
\ENDFOR
\end{algorithmic}
\caption{Adaptive Online Balanced Descent (AOBD)}
\label{alg:mobd}
\end{algorithm}

Algorithm \ref{alg:mobd}, named \emph{Adaptive Online Balanced Descent} (AOBD), describes our proposed algorithm. 
AOBD has two confidence hyperparameters $\bar{\beta} \geq \underline{\beta} > 0$, which represent the confidence of the algorithm designer in the accuracy of the predictions. A higher ratio of $\bar{\beta} / \underline{\beta}$ means high confidence in the predictions, whereas a ratio of $\bar{\beta} / \underline{\beta}$ close to one means less confidence in the predictions.
Note that both line \ref{line:optpr1} and \ref{line:optpr2} represent a single step in Online Balanced Descent \cite{chen2018smoothed} and can be efficiently solved by a binary search. 

The benefit of adopting Online Balanced Descent as the robust algorithm in this context is that there is a close relationship between the switching cost and the hitting cost. For example, if the algorithm were to move to $x(\underline{\beta})$, then the switching cost at this step is exactly equal to $\underline{\beta}$ times the hitting cost. In fact, by convexity, for any $x = x(\lambda)$ for $\lambda \in [\underline{\lambda}, \bar{\lambda}]$ the switching cost is $\beta$ times the hitting cost for a $\beta \in [\underline{\beta}, \bar{\beta}]$. Intuitively, since this is the basis of Online Balanced Descent, this guarantees robustness. Line \ref{line:chooselambda} then chooses $\lambda$ so as to minimize the distance to the advice, ensuring consistency.

The following theorem characterizes the competitive ratio of the AOBD algorithm.

\begin{theorem}
\label{thm:onedim}
Assume that $2 \underline{\beta} \bar{\beta} + \underline{\beta} \geq 1$. Let $\textsc{CR}(\eta)$ be the competitive ratio of AOBD. Then,
\begin{equation}
    \textsc{CR}(\eta) \leq \min\{ \left( 1 + (2 + \bar{\beta}^{-1}) \underline{\beta} \right) (1 + 2 \eta), 1 + (2 + \underline{\beta}^{-1}) \bar{\beta}\}.
\end{equation}
In particular, if $\bar{\beta} = 1 / \delta$ and $\underline{\beta} = \delta / (2 + \delta)$ for $\delta \leq 2$, then,
\begin{equation}
    \textsc{CR}(\eta) \leq \min\left\{ (1 + \delta) (1 + 2 \eta), 1 + 3/\delta + 2/\delta^2 \right\}.
\end{equation}
\end{theorem}

Note that if the predictions are accurate, i.e., if $\eta = 0$, then the competitive ratio of  AOBD is at most $1 + \delta$. This means that AOBD is $(1 + \delta)$-consistent. Moreover, even when $\eta = \infty$, the competitive ratio is uniformly bounded. Hence, AOBD is robust. Also, the algorithm smoothly interpolates to any scenario in between and depends linearly on $\eta$. As before, the confidence hyperparameter $\delta$ characterizes the trade-off between consistency and robustness. 

The trade-off in Theorem \ref{thm:onedim} is in fact necessary, even in this special case, as the next theorem shows.

\begin{theorem}
\label{thm:lbonedim}
Fix any $0 < \delta < 1/2$. Let $\mathcal{A}$ be any deterministic algorithm for the convex, one-dimensional optimization problem. If $\mathcal{A}$ is $(1 + \delta)$-consistent, then $\mathcal{A}$ is at least $1 / (2 \delta)$-robust.
\end{theorem}

Theorem \ref{thm:lbonedim} is inconclusive about whether there exists an algorithm that has a slightly better trade-off between consistency and robustness in terms of $\delta$ than AOBD due to the $2/\delta^2$ term in Theorem \ref{thm:onedim}. Still, Theorem \ref{thm:lbonedim} implies that, even in this special case, no algorithm can be $1$-consistent and robust.
The proofs of Theorem \ref{thm:onedim} and \ref{thm:lbonedim} are provided in Appendices~\ref{sec:proofonedim} and \ref{sec:prooflbonedim}, respectively. 

\section{Numerical Experiments \label{sec:numerics}}

In this section, we empirically evaluate the performance of AOS (Algorithm \ref{alg:main}) in robustifying the decisions made by a machine-learned algorithm for microgrid operation. Specifically, we consider a model of a DNS server receiving power from an islanded microgrid with an 8 MW wind turbine and six 2 MW dispatchable generators. At each time $t = 1, \ldots, T$, the server incurs DNS traffic $d_t \in [0, 12]$ (assumed to be reported in MW of equivalent required power) and the wind turbine produces power $w_t \in [0, 8]$ (in MW). In response, the microgrid operator must choose some generator dispatch $u_t \in \{0, 1\}^6$, where $u_{i, t}$ (the $i$\textsuperscript{th} entry of $u_t$) indicates the commitment status of generator $i$ at time $t$: if $u_{i, t} = 1$, then generator $i$ is producing 2 MW at time $t$, and if $u_{i, t} = 0$, then generator $i$ is off at time $t$.\footnote{In practice, many dispatchable generation resources and in particular small backup generators are constrained in their output flexibility -- i.e., if $u_{i, t} = 1$, then generator $i$ is producing at or near its maximal capacity. As such, it suffices in this example to assume that if $u_{i, t} = 1$, then generator $i$ is producing at exactly 2 MW, though our model formulation can be extended to the case with both discrete commitment variables $u_t$ as well as an auxiliary variable corresponding to generation level.} The goal of the microgrid operator is to minimize costs while approximately balancing dispatchable generation with net load $\ell_t \coloneqq d_t - w_t$, i.e., achieving $\mathbf{2}^\top u_t \approx \ell_t$ at each time $t$. More specifically, at each time $t$, the operator faces the following hitting cost:
$$f_t(u_t) = \mathbf{c}^\top u_t + \gamma \max\{\ell_t - \mathbf{2}^\top u_t, 0\}$$
where $\mathbf{c} \in \R_+^6$ is a vector of fuel costs for each generator, and $\gamma$ is the cost associated with failing to meet 1 MW of demand. Thus $f_t$ has two parts: the first term gives the fuel cost associated with a commitment decision $u_t$, and the second term penalizes mismatch between generation and demand. We assume that $\mathbf{c}$ has no two entries identical, i.e., there is a strict ordering of generator costs, and moreover that $\gamma$ is strictly greater than each entry of $\mathbf{c}$. 

In addition to the hitting cost at each time, there is also a switching cost associated with changing generators' commitment statuses between timesteps, given by
$$\lVert u_t - u_{t-1} \rVert= \beta\|u_t - u_{t-1}\|_{\ell^1}.$$
Such a switching cost is common in the power systems literature on unit commitment, where it is referred to as a cycling or startup/shutdown cost, and reflects fixed costs associated with cycling a generator on/off, including fuel costs for ``cold starts'' and wear and tear due to thermal stress, e.g., \cite{pang_unit_commitment_1976, kumar2012power}. As a result, the total cost faced by the microgrid operator is $\sum_{t=1}^T f_t(u_t) + \lVert u_t - u_{t-1} -$; so the microgrid operator must balance fuel costs and meeting load while taking care not to cycle the generators more than is necessary. Under the assumption that net demand $\ell_t$ is drawn randomly from some probability density, then there is some $\alpha > 0$, unknown to the microgrid operator \emph{a priori}, such that $f_t$ is almost surely $\alpha$-polyhedral with respect to the distance $\lVert \cdot \rVert$ at each time $t$. We can thus frame the problem of optimal microgrid operation in our framework of online optimization with switching costs and non-convex hitting costs, where the non-convexity arises from the discrete nature of the dispatch decisions $u_t$.

In our experiments, we set $\mathbf{c} = (1, 1.2, 1.4, 1.6, 1.8, 2)$, $\gamma = 3$, and $\beta = 8$. We use four days of per-second DNS traffic data from a campus network from \cite{9ync-vv09-19}, which we aggregate into 15 minute periods and suitably normalize. We gather simulated data on the wind speed at an altitude of 100 m off the coast of California  at (39.970406, -128.77481) every 15 minutes of the year 2019 from the Wind Integration National Dataset Toolkit \cite{king2014validation}. We then convert the wind speed data to generation levels using the power curve for an IEC Class 2 wind turbine \cite{king2014validation}. As we have significantly more wind generation data than DNS traffic data, we compute 92 four-day net load trajectories, each using the same window of DNS traffic data but different windows of wind generation. We separate these into a training set of 70 trajectories and a test set of 22 trajectories.

We train a three-layer neural network for the microgrid operation problem as outlined. We assume that at time $t$, in addition to the previous dispatch $u_{t-1}$ and the current net load $\ell_t$, the microgrid operator also has predictions of the next $W = 10$ net loads $\hat{\ell}_{t+1}\, \ldots, \hat{\ell}_{t+W}$, which are used in conjunction with $u_{t-1}$ and $\ell_t$ as inputs to the neural network. We train the model under the assumption that predictions are perfect, i.e., $\hat{\ell}_{t+k} = \ell_{t+k}$ for each $k = 1, \ldots, W$. However, at test time, we evaluate the impact of perturbing the predictions in order to evaluate the success of AOS (Algorithm \ref{alg:main}) at ensuring worst-case robustness even under the presence of distribution shift or increased prediction noise. Note that even though the problem setting is non-convex due to the integer variables $u_t$, Algorithm \ref{alg:main} is in practice tractable, as the minimizations defining $v_t$ and $p_t$ at each time can be computed using standard tools for mixed-integer linear programming.

To illustrate the behavior of machine-learned algorithms that enable improved performance when using perfect predictions of future net loads, we plot in Figure \ref{fig:ex_dispatch} the aggregate dispatch power planned by the trained neural network alongside the dispatches chosen by the algorithm of \cite{zhang2021revisiting} that simply chooses the minimizer $v_t$ of $f_t$ at each time $t$. Since generator commitments are discrete, the microgrid generation can only occur in factors of 2 MW. From the figure, it is evident that the algorithm that follows the minimizers $v_t$ undergoes significant variation in dispatches and hence incurs large costs due to frequent cycling, whereas the machine-learned algorithm, with access to perfect predictions of future demands over the next 10 time periods, has learned to ``smooth'' its dispatches so as not to incur cycling costs unless necessary.

\begin{figure}
    \centering
    \includegraphics[width=0.6\textwidth]{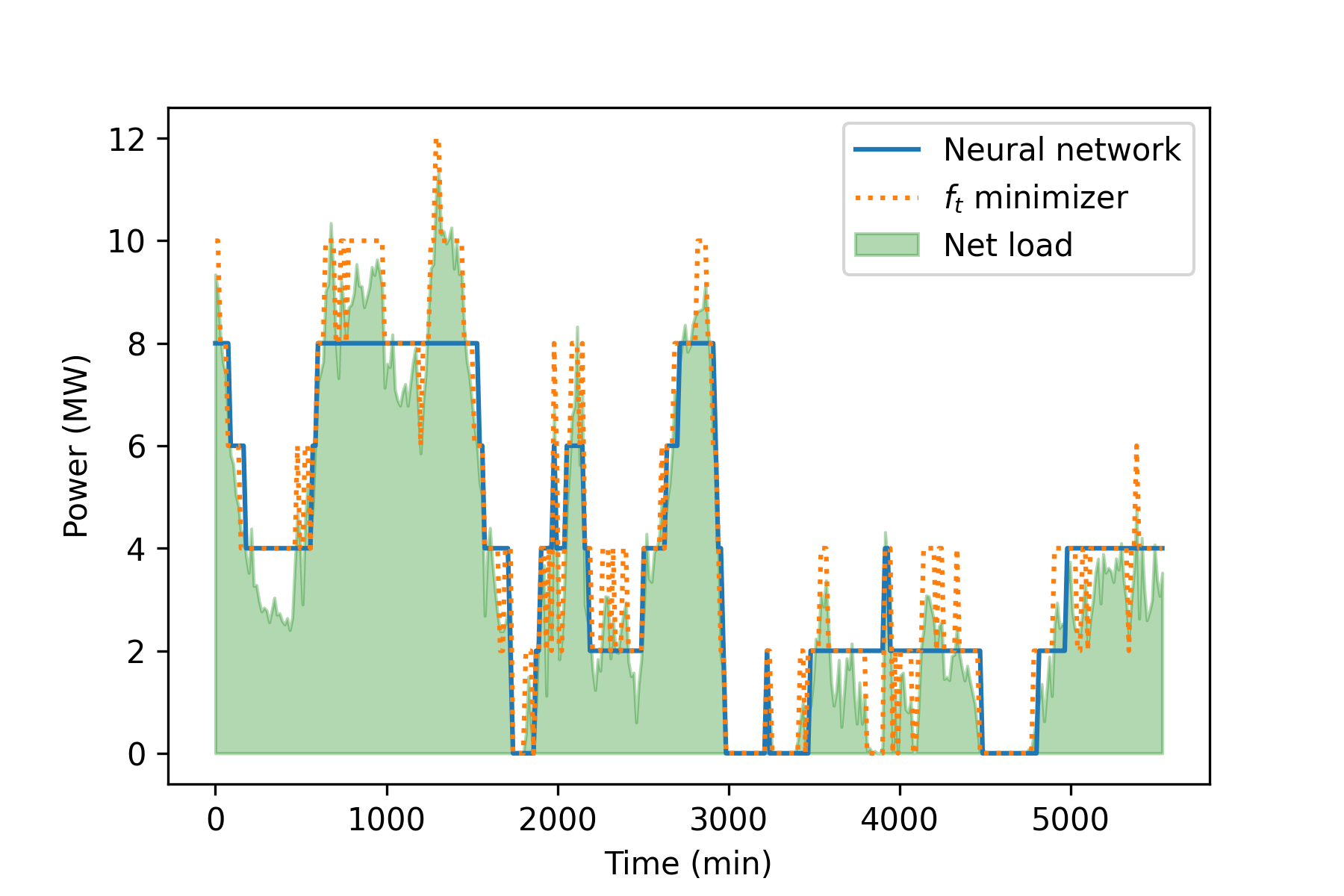}
    \caption{Comparison of the neural network dispatch vs. $f_t$ minimizer dispatch on an example net load trajectory in the test set.}
    \label{fig:ex_dispatch}
\end{figure}

Next, we evaluate the performance of AOS (Algorithm \ref{alg:main}) using the neural network's decision at time $t$ as the prediction $\tilde{x}_t$. We consider the behavior of AOS for various values of the hyperparameter $\delta$, and under two different prediction scenarios: in the first, we assume that predictions are perturbed by zero-mean i.i.d. Gaussian noise with standard deviation $\sigma$, representing the case where predictions at test time are noisier than those during training; and in the second, we assume that predictions are perturbed deterministically with a bias of value $\mu$ that is uniform across predictions at the current timestep, representing the scenario in which the model giving future net load predictions is biased due to distribution shift between training and test time. In each scenario, we vary the values of $\sigma$ and $\mu$ to evaluate the resulting impact on the neural network performance, and resultingly, on the performance of AOS.

We detail the results in the zero-mean Gaussian noise case in Figure \ref{fig:noise_performance}. It is clear that the neural network algorithm performs well in the regime of small $\sigma$, and improves upon the performance of the $f_t$ minimizer algorithm for all $\sigma \leq 0.5$. Moreover, when $\delta$ and $\sigma$ are both small, the performance of AOS nearly matches that of the neural network, and degrades as $\delta$ becomes larger. However, remarkably, the performance of AOS is vastly improved over that of the neural network in the regime of high prediction noise, and even matches or improves upon that of the $f_t$ minimizer algorithm in this regime, even when $\delta = 0.01$. This good performance is likely due both to the adaptive switching of the AOS algorithm as well as the filtering of the neural network decisions $\tilde{x}_t$ to arrive at improved decisions $p_t$. Together, these two properties allow AOS to exploit each algorithm when it is performing well, enabling AOS to achieve the ``best of both worlds'' in terms of performance. Thus this example demonstrates that AOS with a small $\delta$ can bridge the performance of the neural network algorithm and the $f_t$ minimizer algorithm across noise regimes.

\begin{figure}
\centering
\begin{subfigure}[b]{0.49\textwidth}
\centering
\includegraphics[width=\textwidth]{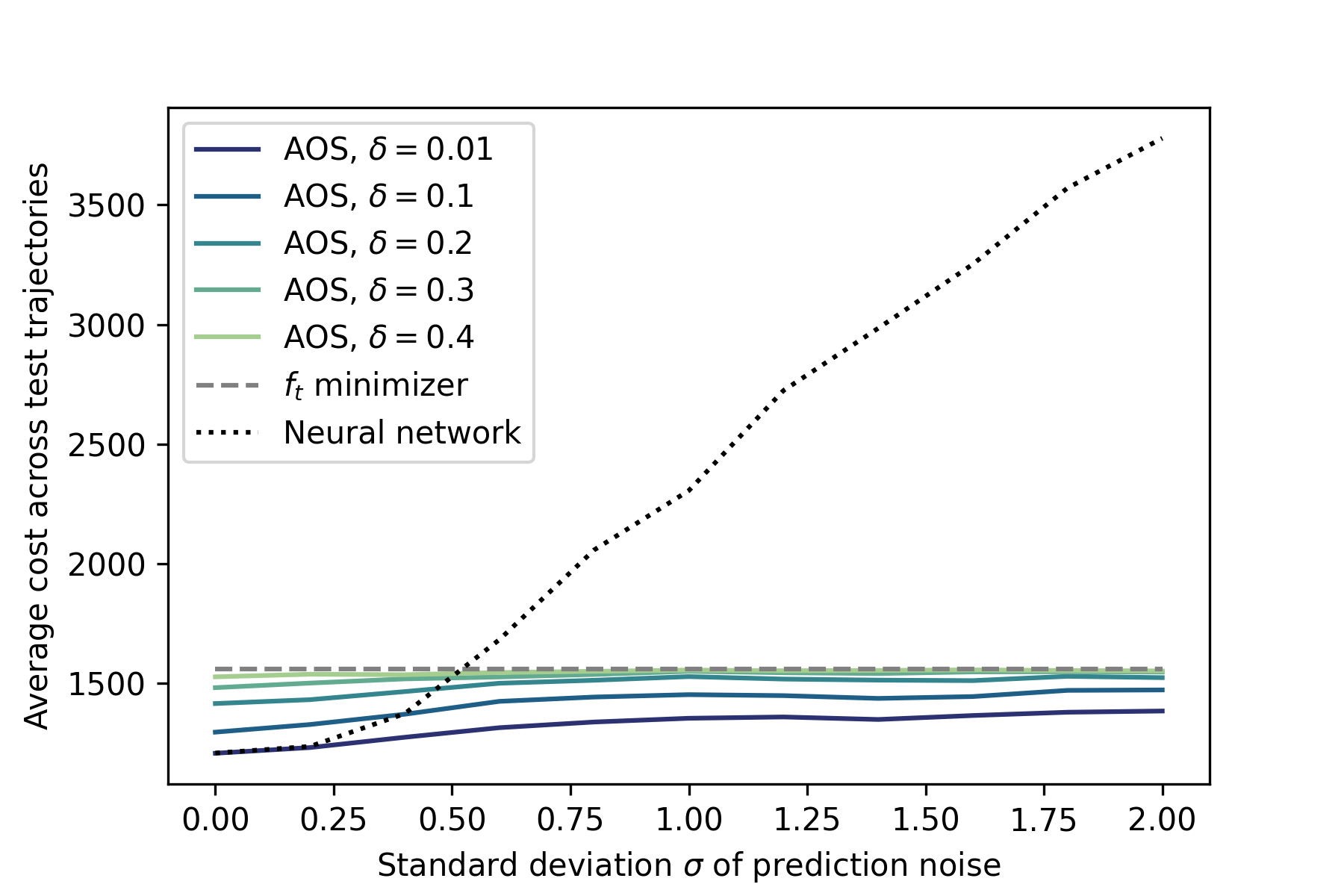}
    \caption{Performance of the AOS algorithm vs. the $f_t$ minimizer and neural network algorithms under different scenarios of prediction noise magnitude and selections of hyperparameter $\delta$.}
    \label{fig:noise_performance}
\end{subfigure}
\hfill
\begin{subfigure}[b]{0.49\textwidth}
\centering
\includegraphics[width=\textwidth]{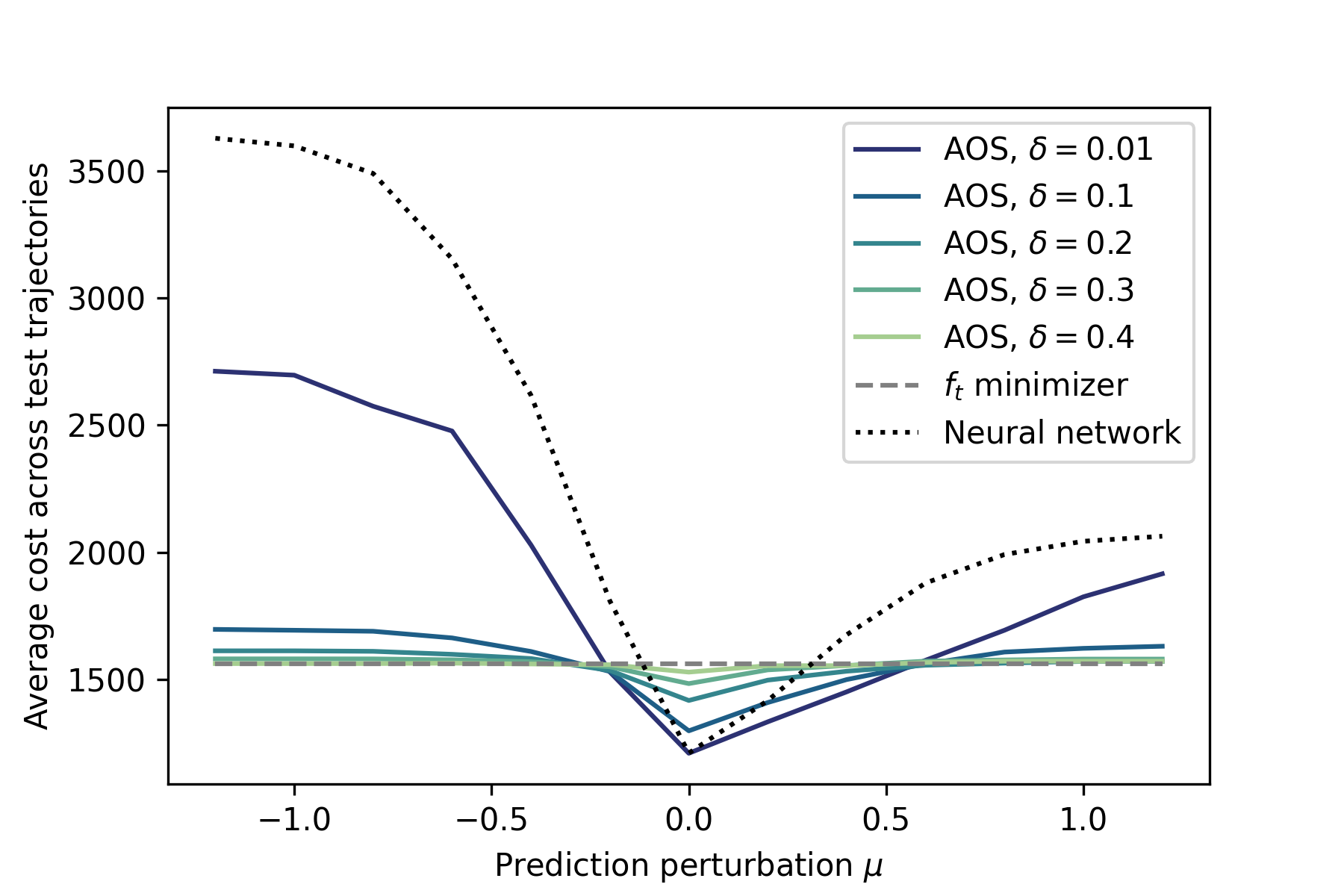}
\caption{Performance of the AOS algorithm vs. the $f_t$ minimizer and neural network algorithms under different scenarios of prediction perturbation and selections of hyperparameter $\delta$.}
\label{fig:perturb_performance}
\end{subfigure}
\end{figure}

We illustrate the performance of the AOS algorithm with various choices of $\delta$ against the neural network and $f_t$ minimizer in the deterministic prediction perturbation scenario in Figure \ref{fig:perturb_performance}. Here, the impact on neural network and AOS performance of perturbation $\mu$ being far from zero is more pronounced: for $|\mu| \approx 1$, the neural network algorithm incurs high cost, and as a result AOS incurs cost that is significantly larger than that of the $f_t$ minimizer algorithm when $\delta = 0.01$. However, as $\delta$ increases, the performance of AOS dramatically improves, rapidly approaching performance comparable to the $f_t$ minimizer. Moreover, in the $|\mu| \ll 1$ regime, we see once again that the neural network outperforms the $f_t$ minimizer, and AOS can closely match the neural network performance when $\delta$ is small. In particular, the choosing $\delta = 0.1$ seems to effectively bridge the good performance of the neural network algorithm in the $|\mu| \ll 1$ setting with the robustness of the $f_t$ minimizer in the large $|\mu|$ setting.

\section{Related Work}

This paper connects and contributes to three growing literatures: (i) online optimization without predictions, (ii) online optimization with trusted predictions, and (iii) online algorithms with untrusted predictions.  We discuss each of these in the following. \\

\noindent
\textbf{Online Optimization without Predictions.} The problem of ``smoothed'' online convex optimization was first introduced by Lin et al.~in (\citeyear{lin2012dynamic}) in the case of a one-dimensional action space.  The one-dimensional version of the problem was further studied by Bansal et al.~(\citeyear{bansal2015competitive}), who designed a $2$-competitive algorithm and later by Antoniadis and Schewior (\citeyear{antoniadis2017tight}), who prove a matching lower bound. 
Andrew et al.~(\citeyear{andrew2013tale}) investigate the compatibility between competitive ratio and sublinear regret, and prove that no algorithm may achieve both simultaneously.
Beyond the one-dimensional setting, Chen et al.~(\citeyear{chen2018smoothed}) prove that in $\R^d$ for general convex functions, the competitive ratio of any online algorithm is at least $\Omega(\sqrt{d})$ with $\ell_2$ switching costs and at least $\Omega(d)$ with $\ell_\infty$ switching costs. The authors introduce the Online Balanced Descent (OBD) algorithm which achieves a competitive ratio of $3 + O(1/\alpha)$ for $\alpha$-polyhedral functions and $\ell_2$ switching costs. 
Later, Goel and Wierman (\citeyear{goel2019online}) prove that OBD is also $3 + O(1/\alpha)$ competitive if the switching cost is the squared $\ell_2$ norm and the hitting cost functions $f_t$ are $m$-strongly convex. In this context, the authors prove that any online algorithm must have a competitive ratio of at least $\Omega(m^{-1/2})$ and introduce two algorithms, Greedy OBD (G-OBD) and Regularized OBD (R-OBD), that both have a competitive ratio of $O(m^{-1/2})$.

Note that all results in the case of online \emph{non-convex} optimization are very limited. Lin et al.~(\citeyear{lin2020online}) provide an algorithm, in the setting with no perfect predictions of future hitting costs, that greedily follows the minimizers and achieves competitive ratio $1 + 2 / \alpha$. Later, Zhang et al.~(\citeyear{zhang2021revisiting}) refined this result, proving that the greedy algorithm is $\max\{1, 2 / \alpha\}$-competitive. \\

\noindent
\textbf{Online Optimization with Trusted Predictions.} A separate but related line of work seeks to provide improved competitive bounds through the use of trusted predictions.  Typically, such predictions are assumed to have no error, e.g., \cite{li2018using, lin2020online}, but some papers assume known stochastic predictions error, e.g., \cite{chen2015online, chen2016using, comden2019online}. 
An early paper in this area by Chen et al.~(\citeyear{chen2015online}) introduces the most common prediction error model (a colored noise model). 
Chen et al.~(\citeyear{chen2016using}) consider a similar stochastic prediction model, but weaken the assumptions on the hitting cost functions.
A different style of algorithm was introduced by Li et al.~(\citeyear{li2018using}), which assumes that the next $w$ functions are known perfectly.
Later, Li and Li~(\citeyear{li2020leveraging}) generalize this to imperfect predictions. 

The case of online non-convex optimization has proven difficult and the only result we are aware of is from Lin et al.~(\citeyear{lin2020online}), which assumes perfect predictions and provides two sufficient conditions on the structure of $f_t$ and the switching cost (an order of growth and an approximate triangle inequality condition) such that their algorithm, Synchronized Fixed Horizon Control (SFHC), achieves a competitive ratio of $1 + O(1 / w)$. \\

\noindent
\textbf{Online Algorithms with Untrusted Predictions.} This work is the first to look at the use of untrusted predictions in the setting of online optimization, with more recent work such as \cite{christianson2022cbc} building upon our results and techniques in the more restricted setting of \emph{convex} body and function chasing. The framework for studying untrusted predictions that we adapt was introduced by Lykouris and Vassilvitskii (\citeyear{lykouris2018competitive}) in the context of caching and by Mahdian et al. (\citeyear{MNS12}) in the context of load balancing and facility location. At this point, the framework has received significant attention and has been applied to bipartite matching \cite{kumar2018semi}, ski rental problems \cite{KPS18, GP19, WZ20, AGP20, WLW20}, set cover \cite{BMS20}, online scheduling \cite{LLMV20, Mitzenmacher2020scheduling, ALT21}, bloom-filters \cite{mitzenmacher2018model}, frequency estimation \cite{hsu2019learning}, matching and secretary problems \cite{LMRX21, DLPV21, AGKK20, JLTZ21}, capacity scaling in data-centers \cite{BMRS20, rutten2021new}, and more \cite{sun2021pareto, maghakian2021leveraging, lee2021online}.

Recent work in \cite{antoniadisOnlineMetricAlgorithms2020} proposes combining individually robust and consistent algorithms for Metrical Task Systems by applying classic online algorithms for the $k$-server problem \cite{fiat1994competitive} and the $k$-experts problem \cite{blum2000line}. However, their results do not yield tight bounds on robustness and consistency in our setting. If $\alpha < 2$, their deterministic algorithm achieves 9-consistency and $\frac{18}{\alpha}$-robustness, and their randomized algorithm achieves $(1+\delta)$-consistency and $\frac{2(1+\delta)}{\alpha}$-robustness in expectation but has a large additive factor of $\mathcal{O}(D/\delta)$ on both bounds, where $D$ is the diameter of the problem instance. However, our main result shows that it is possible to obtain deterministic robustness and near-perfect consistency \emph{without} such a dependence on bounded diameter. 

\section{Concluding Remarks}

This paper considers online non-convex optimization with untrusted predictions. We have introduced Adaptive Online Switching (AOS) and proved that AOS nearly matches the performance of the hindsight optimal if predictions are accurate. At the same time, if feasible, AOS has a worst-case performance bound even if predictions are completely inaccurate. We validate the performance guarantees of AOS with experiments reported in Section \ref{sec:numerics}, and complement our results with a lower bound on the trade-off between consistency and robustness of \emph{any} algorithm, showing that AOS has an order-optimal trade-off. Moreover, we have proven the necessity of memory in AOS and introduced a specialized algorithm, Adaptive Online Balanced Descent, for the special case of one-dimensional, convex optimization.  An open question, which we leave for future work, is that of tight upper and lower bounds on the trade-off between consistency and robustness for online \emph{convex} optimization in general dimensions. Note that any algorithm for online convex optimization must again use memory due to Theorem \ref{thm:memoryless}, but the lower bound in Theorem \ref{thm:lowerbound} might be improved upon in this case. 

\def\UrlBreaks{\do\/\do-}
\bibliographystyle{apalike}
\bibliography{main}

\appendix

\section{Online Convex Optimization with Bregman Divergence Switching Costs}
\label{appendix:soco_squared_dist}

To highlight the challenges for trade-offs between consistency and robustness in related online optimization settings, we prove a lower bound on the trade-off between consistency and robustness for the case of  squared $\ell_2$ switching costs. In particular, we consider the following assumptions.
\begin{enumerate}
    \item The functions $f_t$ are $\alpha$-strongly convex and $\beta$-strongly smooth.
    \item The space $X = \R$ and the distance $\lVert x_{t-1} - x_t \rVert = \frac{1}{2} \lVert x_{t-1} - x_t \rVert_2^2$.
\end{enumerate}
Note that the squared $\ell_2$ switching costs are an instance of the more general Bregman divergence switching costs \cite{goel2019beyond}. We now state our lower bound.
\begin{lemma}
Let $\mathcal{A}$ be any algorithm for the online convex optimization problem. The next two statements are mutually exclusive.
\begin{enumerate}
    \item $\mathcal{A}$ is $\gamma$-robust where $\gamma < \infty$.
    \item $\mathcal{A}$ is $c$-consistent where
    \begin{equation}
    \label{eq:consistentlowerbound}
        c < \frac{1 + \sqrt{1 + 4 \alpha^{-1}}}{2 (1 + \beta^{-1})}.
    \end{equation}
\end{enumerate}
\end{lemma}

Notice that the consistency bound in \eqref{eq:consistentlowerbound} is achieved by an online algorithm which ignores predictions in \cite{goel2019beyond} in the case that $\beta = \infty$. Thus, this bound shows the difficulty in extracting value from untrusted predictions in this setting.

\begin{proof}
Fix any $\delta > 0$. Let $x_0 = 0$ and $f_t(x) = \alpha x^2 / 2$ for $1 \leq t \leq T-1$. The predictions $\tilde{x}_t$ are equal to the hindsight optimal as if $f_T(x) = \beta (x - \delta)^2 / 2$. We distinguish two cases.
\begin{enumerate}
    \item Assume there exists $1 \leq t \leq T-1$ such that $x_t > 0$. Then, let $f_T(x) = \beta x^2 / 2$. The optimal algorithm achieves a cost of zero by letting $o_t = 0$ for all $1 \leq t \leq T$. However, the cost of $\mathcal{A}$ is at least $f_t(x_t) = \alpha x_t^2 / 2 > 0$ and hence the competitive ratio is unbounded. Therefore, $\mathcal{A}$ is not robust.
    \item Assume that $x_t = 0$ for $1 \leq t \leq T$. Then, let $\theta_T = (\delta, \beta)$. Goel et al. \cite{goel2019beyond} prove that the cost of the hindsight optimal is at most $\delta^2 (-\alpha + \sqrt{\alpha^2 + 4 \alpha}) / 4$ asymptotically as $T \to \infty$. The algorithm $\mathcal{A}$ suffers cost at least
    \begin{equation}
        \min_{x \in [0, \delta]} \frac{x^2}{2} + \frac{\beta (x - \delta)^2}{2} = \frac{\delta^2}{2 (1 + \beta^{-1})},
    \end{equation}
    and hence the competitive ratio is at least
    \begin{equation}
    \label{eq:lbconst}
        \frac{2}{(1 + \beta^{-1}) (-\alpha + \sqrt{\alpha^2 + 4 \alpha})} = \frac{1 + \sqrt{1 + 4 \alpha^{-1}}}{2 (1 + \beta^{-1})}.
    \end{equation}
    Therefore, since the predictions are perfect, $\mathcal{A}$ is not $c$-consistent for any $c$ strictly smaller than \eqref{eq:lbconst}.
\end{enumerate}
\end{proof}

\section{Proof of Lemma \ref{lemma:ftpalg}}

The proof of the lemma follows largely along the same lines as the proof in \cite{antoniadisOnlineMetricAlgorithms2020}, but we should be slightly careful as the points $p_t$ are not the true minimizer but rather an approximation thereof.

\begin{proof}[Proof of Lemma \ref{lemma:ftpalg}]
Let
\begin{equation}
\begin{multlined}
    C_t = \sum_{i=1}^t \left( f_i(p_i) + \lVert p_i - p_{i-1} \rVert \right) + \lVert o_{t+1} - p_t \rVert \\ + \left( 1 + \varepsilon(p) \right) \left( f_{t+1}(o_{t+1}) + \sum_{i=t+2}^T \left( f_i(o_i) + \lVert o_i - o_{i-1} \rVert \right) \right),
\end{multlined}
\end{equation}
for $1 \leq t \leq T - 1$ and $C_T = \sum_{i=1}^T \left( f_i(p_i) + \lVert p_i - p_{i-1} \rVert \right)$. Then, by the triangle inequality,
\begin{equation}
\begin{multlined}
    C_t - C_{t-1}
    \leq 2 \lVert o_t - p_t \rVert + f_t(p_t) - \left( 1 + \varepsilon(p) \right) f_t(o_t) \\
    \leq 2 \lVert o_t - \tilde{x}_t \rVert - \left( 1 + \varepsilon(p) \right) f_t(o_t) + 2 \lVert \tilde{x}_t - p_t \rVert + f_t(p_t) \\
    \leq \left( 4 + 2 \varepsilon(p) \right) \lVert o_t - \tilde{x}_t \rVert,
\end{multlined}
\end{equation}
where we used that $p_t$ is the approximate minimizer of \eqref{eq:ftp} in the last inequality. Then, summing the inequality over $t$ and observing that $C_0$ is at most $\left( 1 + \varepsilon(p) \right)$ times the cost of the optimal trajectory whereas $C_T$ is the cost of \textsc{FtP} yields the competitive ratio bound.
\end{proof}

\section{Proof of Theorem \ref{thm:main}}
\label{app:mainproof}

Recall Theorem \ref{thm:main} as stated in the previous section. The theorem summarizes two separate bounds on the competitive ratio of the AOS algorithm (Algorithm \ref{alg:main}), which is represented by the two terms in the minimum. We provide a proof of Theorem \ref{thm:main} here, where we treat the two bounds separately and prove the following two propositions, which are slightly stronger than Theorem \ref{thm:main}. Let $\textsc{Alg}(t) := f_t(x_t) + \lVert x_t - x_{t-1} \rVert$ denote the cost incurred by the AOS algorithm at time $t$.

\begin{proposition} (Consistency)
\label{prop:consistent}
\begin{equation}
    \sum_{i = 1}^T \textsc{Alg}(i) \leq (1 + \delta + \gamma) \sum_{i = 1}^T \textsc{Adv}(i)
\end{equation}
\end{proposition}

\begin{proposition} (Robustness)
\label{prop:competitive}
\begin{equation}
    \sum_{i = 1}^T \textsc{Alg}(i) \leq \left( \frac{4 U(\infty) + 4 + 2 \delta}{\gamma} + 2 U(\infty) + 3 \right) \sum_{i = 1}^T \textsc{Rob}(i) + 2 U(\infty) \sum_{i = 1}^T \lVert v_i - r_i \rVert,
\end{equation}
where $U(\infty) := \sup_{t \geq 1} U(t)$ and $U(t)$ is as defined in \eqref{eq:duallp}. Moreover, assume that $(\alpha, \delta)$ is feasible and $\frac{2}{\alpha \delta} \in \N$. Then,
\begin{equation}
    \sum_{i = 1}^T \textsc{Alg}(i)
    \leq \left( \frac{4 \tilde{U} + 4 + 2 \delta}{\gamma} + 2 \tilde{U} + 3 \right) \sum_{i = 1}^T \textsc{Rob}(i) + 2 \tilde{U} \sum_{i = 1}^T \lVert v_i - r_i \rVert,
\end{equation}
where
\begin{equation}
    \tilde{U} = \alpha \left( \frac{2}{\alpha + \delta (1 + \alpha)} \right)^{2 / \alpha \delta} + \frac{2}{(2 - \alpha - \delta (1 + \alpha))^2} \left( \alpha \left( \frac{2}{\alpha + \delta (1 + \alpha)} \right)^{2 / \alpha \delta} - \frac{2 - \alpha}{\delta} + 1 \right).
\end{equation}
\end{proposition}

We note that the assumption that $2 / (\alpha \delta) \in \N$ in the statement of Proposition~\ref{prop:competitive} is without loss of generality, but it prevents rounding symbols from appearing in the notation.

\begin{proof}[Proof of Theorem \ref{thm:main}]
Note that Proposition \ref{prop:consistent} and Lemma \ref{lemma:ftpalg} imply that
\begin{equation}
    \sum_{i = 1}^T \textsc{Alg}(i) \leq (1 + \delta + \gamma) \sum_{i = 1}^T \textsc{Adv}(i) \leq (1 + \delta + \gamma) (1 + \varepsilon(p) + (4 + 2 \varepsilon(p)) \eta) \sum_{i = 1}^T \textsc{Opt}(i),
\end{equation}
which proves that the consistency of the AOS algorithm. To show the robustness bound, note that
\begin{equation}
\begin{aligned}
    \sum_{i = 1}^T \textsc{Rob}(i)
    &= \sum_{i = 1}^T \left( f_i(r_i) + \lVert r_i - r_{i-1} \rVert \right) \\
    &\leq \sum_{i = 1}^T \left( \frac{f_i(r_i)}{f_i(v_i)} f_i(v_i) + \lVert v_i - v_{i-1} \rVert + \lVert v_i - r_i \rVert + \lVert v_{i-1} - r_{i-1} \rVert \right) \\
    &\leq (1 + \varepsilon(r)) \sum_{i = 1}^T \left( f_i(v_i) + \lVert v_i - v_{i-1} \rVert \right) + \frac{2}{\alpha} \sum_{i = 1}^T \left( f_i(r_i) - f_i(v_i) \right) \\
    &\leq \left( 1 + \varepsilon(r) + \frac{2 \varepsilon(r)}{\alpha} \right) \sum_{i = 1}^T \left( f_i(v_i) + \lVert v_i - v_{i-1} \rVert \right) \\
    &\leq \left( 1 + \varepsilon(r) + \frac{2 \varepsilon(r)}{\alpha} \right) \max\left\{ 1, \frac{2}{\alpha} \right\} \sum_{i = 1}^T \textsc{Opt}(i),
\end{aligned}
\end{equation}
where the first inequality follows by the triangle inequality and the last inequality follows since the algorithm that follows the minimizers has a competitive ratio of $\max\{1, 2 / \alpha\}$. Similarly,
\begin{equation}
\begin{aligned}
    \sum_{i = 1}^T \lVert v_i - r_i \rVert
    &\leq \frac{1}{\alpha} \sum_{i = 1}^T \left( f_i(r_i) - f_i(v_i) \right)  \\
    &\leq \frac{\varepsilon(r)}{\alpha} \sum_{i = 1}^T \left( f_i(v_i) + \lVert v_i - v_{i-1} \rVert \right)
    \leq \frac{\varepsilon(r)}{\alpha} \max\left\{ 1, \frac{2}{\alpha} \right\} \sum_{i = 1}^T \textsc{Opt}(i),
\end{aligned}
\end{equation}
Therefore, Proposition \ref{prop:competitive} implies that
\begin{equation}
\begin{aligned}
    \sum_{i = 1}^T \textsc{Alg}(i)
    &\leq \left( \frac{4 \tilde{U} + 4 + 2 \delta}{\gamma} + 2 \tilde{U} + 3 \right) \sum_{i = 1}^T \left( \textsc{Rob}(i) + \lVert v_i - r_i \rVert \right) \\
    &\leq \left( \frac{4 \tilde{U} + 4 + 2 \delta}{\gamma} + 2 \tilde{U} + 3 \right) \left( 1 + \varepsilon(r) + \frac{2\varepsilon(r)}{\alpha} \right) \max\left\{ 1, \frac{2}{\alpha} \right\} \sum_{i = 1}^T \textsc{Opt}(i),
\end{aligned}
\end{equation}
which completes the proof of the theorem.
\end{proof}

We detail the proofs of Proposition \ref{prop:consistent} and \ref{prop:competitive} in the next two sections, respectively.

\subsection{Proof of Consistency Bound}

Recall that the AOS algorithm works in cycles or \emph{stages}. A stage $k \in \N$ starts at time $T_k$ when the algorithm switches to the advice and ends at the time $T_{k+1}-1$ when the algorithm again switches back to  the advice after having switched to the greedy algorithm in between. 
Also, recall that the time when the algorithm switches to the greedy algorithm in the $k$-th stage is denoted by $M_k$.
In the rest of the proofs, whenever we refer to $T_k$ (and $M_k$), we will implicitly assume that $k$ is such that $T_k \leq T$ (and $M_k \leq T$, respectively).
The lemma below proves that the AOS algorithm is $(1 + \delta + \gamma)$-consistent within each stage.

\begin{lemma}
\label{lemma:consistent}
For all $k \in \N$ and $T_k \leq t \leq T_{k+1}-1$,
\begin{equation}
    \sum_{i = T_k}^t \textsc{Alg}(i) - \lVert r_{T_k-1} - p_{T_k-1} \rVert
    \leq (1 + \delta + \gamma) \sum_{i = T_k}^t \textsc{Adv}(i) - \lVert x_t - p_t \rVert.
\end{equation}
\end{lemma}

Lemma \ref{lemma:consistent} readily implies the consistency bound stated in Proposition \ref{prop:consistent}.

\begin{proof}[Proof of Proposition \ref{prop:consistent}]
Let $K \in \N$ be the largest integer such that $T \geq T_K$.
Then,
\begin{equation}
\begin{aligned}
    \sum_{i = 1}^T \textsc{Alg}(i)
    &\leq \sum_{k = 1}^{K-1} \left( \sum_{i = T_k}^{T_{k+1}-1} \textsc{Alg}(i) + \lVert r_{T_{k+1}-1} - p_{T_{k+1}-1} \rVert - \lVert r_{T_k-1} - p_{T_k-1} \rVert \right) \\ &+ \sum_{i = T_K}^T \textsc{Alg}(i) + \lVert x_T - p_T \rVert - \lVert r_{T_K-1} - p_{T_K-1} \rVert
    \leq (1 + \delta + \gamma) \sum_{i = 1}^T \textsc{Adv}(i),
\end{aligned}
\end{equation}
where the first inequality follows from the telescoping sum and the second inequality follows using Lemma \ref{lemma:consistent}.
\end{proof}

We now provide the proof of Lemma \ref{lemma:consistent}. The proof is rather straightforward and follows along similar lines as the intuition introduced before. The proof depends on the conditions in lines \ref{line:whileadv} and \ref{line:whilerob} in the AOS algorithm.

\begin{proof}[Proof of Lemma \ref{lemma:consistent}]
Fix any $k \in \N$ and $T_k \leq t \leq M_k-1$. Then,
\begin{equation}
\begin{aligned}
    \sum_{i = T_k}^t \textsc{Alg}(i) - \lVert r_{T_k-1} - p_{T_k-1} \rVert
    &= f_{T_k}(p_{T_k}) + \lVert r_{T_k-1} - p_{T_k} \rVert + \sum_{i = T_k+1}^t \textsc{Adv}(i) - \lVert r_{T_k-1} - p_{T_k-1} \rVert \\
    &\leq \sum_{i = T_k}^t \textsc{Adv}(i) - \lVert p_t - p_t \rVert,
\end{aligned}
\end{equation}
where the inequality follows by the triangle inequality and $x_t = p_t$. This proves the lemma for $T_k \leq t \leq M_k - 1$.
Fix any $k \in \N$ and $M_k \leq t \leq T_{k+1}-1$. Then,
\begin{equation}
\begin{aligned}
    \sum_{i = T_k}^t& \textsc{Alg}(i) - \lVert r_{T_k-1} - p_{T_k-1} \rVert \\
    &= f_{T_k}(p_{T_k}) + \lVert r_{T_k-1} - p_{T_k} \rVert + \sum_{i = T_k+1}^{M_k-1} \textsc{Adv}(i) + \sum_{i = M_k}^t \textsc{Alg}(i) - \lVert r_{T_k-1} - p_{T_k-1} \rVert \\
    &\leq \sum_{i = T_k}^{M_k-1} \textsc{Adv}(i) + f_{M_k}(r_{M_k}) + \lVert p_{M_k-1} - r_{M_k} \rVert + \sum_{i = M_k+1}^t \textsc{Rob}(i) \\
    &\leq \sum_{i = T_k}^{M_k-1} \textsc{Adv}(i) + \lVert p_{M_k-1} - r_{M_k-1} \rVert + \sum_{i = M_k}^t \textsc{Rob}(i) \\
    &\leq (1 + \delta) \sum_{i = T_k}^{M_k} \textsc{Adv}(i) + \sum_{i = M_k+1}^t \textsc{Rob}(i) - \lVert r_{M_k} - p_{M_k} \rVert 
    \leq (1 + \delta + \gamma) \sum_{i = T_k}^t \textsc{Adv}(i) - \lVert r_t - p_t \rVert,
\end{aligned}
\end{equation}
where the first and second inequality follow by the triangle inequality, the third inequality follows by line \ref{line:whileadv} of Algorithm \ref{alg:main}, and the fourth inequality follows by line \ref{line:whilerob} of Algorithm \ref{alg:main} and $x_t = r_t$.
\end{proof}

\subsection{Proof of Robustness Bound}

\begin{lemma}
\label{lemma:competitive}
For all $k \in \N$ and $1 \leq t \leq M_k-T_k-1$,
\begin{equation}
    \sum_{i = T_k+1}^{T_k+t} \textsc{Adv}(i) \leq U(t) \left( \sum_{i = T_k+1}^{T_k+t} \left( \textsc{Rob}(i) + 2 \lVert v_i - r_i \rVert \right) + \lVert p_{T_k} - r_{T_k} \rVert \right),
\end{equation}
where $U(t)$ is as defined in Proposition \ref{prop:competitive}.
\end{lemma}

\begin{proof}[Proof of Lemma \ref{lemma:competitive}]
The idea of the proof is to write the inequality as an optimization problem, where the adversary aims to invalidate the inequality by maximizing the left-hand side and minimizing the right-hand side simultaneously. We reduce the constraints in the optimization problem to a set of necessary, linear constraints, and therefore construct a linear program that represents the inequality. The objective value of the linear program is equal to the objective value of its dual, which results in the expression for $U(t)$ in~\eqref{eq:duallp}.

Fix any $k \in \N$ and $1 \leq t \leq M_k - T_k - 1$. Let
\begin{equation}
\begin{aligned}
\label{eq:defqts}
    \Delta_s^\parallel &:= \lVert p_{T_k+s} - v_{T_k+s} \rVert - \lVert p_{T_k+s-1} - v_{T_k+s-1} \rVert &\text{ for } 2 \leq s \leq t, \\
    \Delta_1^\parallel &:= \lVert p_{T_k+1} - v_{T_k+1} \rVert, &\\
    \Delta_s^\perp &:= \lVert p_{T_k+s} - p_{T_k+s-1} \rVert - \left\lvert \Delta_s^\parallel \right\rvert &\text{ for } 1 \leq s \leq t, \\
    \sigma_s &:= \frac{f_{T_k+s}(p_{T_k+s})}{\lVert p_{T_k+s} - v_{T_k+s} \rVert} &\text{ for } 1 \leq s \leq t, \\
    \gamma_s &:= \textsc{Rob}(T_k+s) + \lVert v_{T_k+s} - r_{T_k+s} \rVert + \lVert v_{T_k+s-1} - r_{T_k+s-1} \rVert &\text{ for } 2 \leq s \leq t, \\
    \gamma_1 &:= \textsc{Rob}(T_k+1) + \lVert v_{T_k+1} - r_{T_k+1} \rVert + \lVert p_{T_k} - r_{T_k} \rVert.
\end{aligned}
\end{equation}
Note that the problem may be rewritten only in terms of the definitions introduced above. In fact,
\begin{equation}
\begin{aligned}
    \sum_{i = T_k+1}^{T_k+t} \textsc{Adv}(i) - U(t) \left( \sum_{i = T_k+1}^{T_k+t} \left( \textsc{Rob}(i) + 2 \lVert v_i - r_i \rVert \right) + \lVert p_{T_k} - r_{T_k} \rVert \right) \\
    \leq \sum_{i = 1}^t \left( \Delta_i^\perp + \left\lvert \Delta_i^\parallel \right\rvert + \sigma_i \sum_{j = 1}^i \Delta_j^\parallel \right) - U(t) \sum_{i = 1}^t \gamma_i.
\end{aligned}
\end{equation}
We identify a set of necessary constraints on the quantities introduced above. Note that
\begin{equation}
\begin{aligned}
    \sum_{i = T_k}^{T_k+s} \textsc{Adv}(i) &+ 2 \textsc{Rob}(T_k+s) + 2 \lVert v_{T_k+s-1} - r_{T_k+s-1} \rVert + 2 \lVert p_{T_k+s-1} - v_{T_k+s-1} \rVert \\
    &\geq \sum_{i = T_k}^{T_k+s} \textsc{Adv}(i) + 2 \textsc{Rob}(T_k+s) + 2 \lVert p_{T_k+s-1} - r_{T_k+s-1} \rVert \\
    &\geq \sum_{i = T_k}^{T_k+s-1} \textsc{Adv}(i) + \textsc{Rob}(T_k+s) + \lVert p_{T_k+s-1} - v_{T_k+s-1} \rVert + \alpha \lVert p_{T_k+s} - r_{T_k+s} \rVert \\
    &\qquad+ \lVert p_{T_k+s} - p_{T_k+s-1} \rVert + \lVert r_{T_k+s} - r_{T_k+s-1} \rVert + \lVert p_{T_k+s-1} - r_{T_k+s-1} \rVert \\
    &\geq \sum_{i = T_k}^{T_k+s-1} \textsc{Adv}(i) + \textsc{Rob}(T_k+s) + \lVert p_{T_k+s-1} - r_{T_k+s-1} \rVert + \lVert r_{T_k+s} - p_{T_k+s} \rVert \\
    &\qquad+ \alpha \lVert p_{T_k+s} - v_{T_k+s} \rVert \\
    &\geq (1 + \delta) \sum_{i = T_k}^{T_k+s} \textsc{Adv}(i) + \alpha \lVert p_{T_k+s} - v_{T_k+s} \rVert,
\end{aligned}
\end{equation}
for all $1 \leq s \leq t$, where the second inequality follows by the triangle inequality and the third inequality follows from line \ref{line:whileadv} of Algorithm \ref{alg:main}. Therefore,
\begin{equation}
\label{eq:constraint1}
    2 \gamma_s + 2 \sum_{i = 1}^{s-1} \Delta_i^\parallel
    \geq \delta \sum_{i = 1}^s \left( \Delta_i^\perp + \left\lvert \Delta_i^\parallel \right\rvert + \sigma_i \sum_{j = 1}^i \Delta_j^\parallel \right) + \alpha \sum_{i = 1}^s \Delta_i^\parallel \text{ for } 1 \leq s \leq t.
\end{equation}
The next constraints follow directly from the definitions:
\begin{equation}
\label{eq:constraint2}
    \sum_{i = 1}^s \Delta_i^\parallel \geq 0, \;
    \Delta_s^\perp \geq -\gamma_s, \;
    \sigma_s \geq \alpha, \;
    \gamma_s \geq 0 \text{ for } 1 \leq s \leq t,
\end{equation}
where the second constraint follows by the triangle inequality. Now,
\begin{equation}
\begin{aligned}
\label{eq:optproblem}
    \sum_{i = T_k+1}^{T_k+t} \textsc{Adv}(i) - U(t) &\Bigg( \sum_{i = T_k+1}^{T_k+t} \left( \textsc{Rob}(i) + 2 \lVert v_i - r_i \rVert \right) + \lVert p_{T_k} - r_{T_k} \rVert \Bigg) \\
    \leq \sup_{\Delta^\parallel, \Delta^\perp, \sigma, \gamma} \quad &\sum_{i = 1}^t \left( \Delta_i^\perp + \left\lvert \Delta_i^\parallel \right\rvert + \sigma_s \sum_{j = 1}^i \Delta_j^\parallel \right) - U(t) \sum_{i = 1}^t \gamma_i \\
    &\text{s.t.} \quad \text{\eqref{eq:constraint1} and \eqref{eq:constraint2}}.
\end{aligned}
\end{equation}
We claim that \eqref{eq:optproblem} is equivalent to a linear problem. More formally, we will prove the following:
\begin{claim}
\label{claim:linproblem}
Fix any $\lambda = (\Delta^\parallel, \Delta^\perp, \sigma, \gamma)$, which satisfies \eqref{eq:constraint1} and \eqref{eq:constraint2}. Then, there exists $\hat{\lambda} = (\hat{\Delta}^\parallel, \hat{\Delta}^\perp, \alpha, \hat{\gamma})$, satisfying \eqref{eq:constraint1} and \eqref{eq:constraint2}, such that $\hat{\Delta}_s^\parallel \geq 0$ for $1 \leq s \leq t-1$, $\hat{\Delta}_t^\parallel \leq 0$, and the objective values of $\lambda$ and $\hat{\lambda}$ are equal.
\end{claim}
We provide the proof of Claim \ref{claim:linproblem} below. Note that Claim \ref{claim:linproblem} implies that
\begin{equation}
\begin{aligned}
\label{eq:linproblem}
    \text{\eqref{eq:optproblem} } = \max_{\Delta^\parallel, \Delta^\perp, \gamma} \quad &\sum_{i = 1}^{t-1} \left( \Delta_i^\perp + \Delta_i^\parallel + \alpha \sum_{j = 1}^i \Delta_j^\parallel \right) + \left( \Delta_t^\perp + \Delta_t^\parallel + \alpha \left( \sum_{i = 1}^{t-1} \Delta_i^\parallel - \Delta_t^\parallel \right) \right) - U(t) \sum_{i = 1}^t \gamma_i \\
    \text{s.t.} \quad 2 \gamma_s + 2 \sum_{i = 1}^{s-1} \Delta_i^\parallel &\geq \delta \sum_{i = 1}^s \left( \Delta_i^\perp + \Delta_i^\parallel + \alpha \sum_{j = 1}^i \Delta_j^\parallel \right) + \alpha \sum_{i = 1}^s \Delta_i^\parallel \text{ for } 1 \leq s \leq t-1, \\
    2 \gamma_t + 2 \sum_{i = 1}^{t-1} \Delta_i^\parallel &\geq \delta \sum_{i = 1}^{t-1} \left( \Delta_i^\perp + \Delta_i^\parallel + \alpha \sum_{j = 1}^i \Delta_j^\parallel \right) + \delta \left( \Delta_t^\perp + \Delta_t^\parallel + \alpha \left( \sum_{i = 1}^{t-1} \Delta_i^\parallel - \Delta_t^\parallel \right) \right), \\
    \Delta_s^\parallel &\geq 0, \; \; \Delta_s^\perp \geq -\gamma_s, \; \gamma_s \geq 0 \text{ for } 1 \leq s \leq t.
\end{aligned}
\end{equation}
Observe that the linear program in \eqref{eq:linproblem} is conic. Hence, the objective value of \eqref{eq:linproblem} is either zero or unbounded. Then, by strong duality, the objective value is equal to zero if and only if there exists a solution $y$ to the dual, i.e.,
\begin{equation}
\begin{aligned}
\label{eq:dualproblem}
   \delta \sum_{i = s}^t (1 + \alpha (i - s + 1)) y_i + \alpha \sum_{i = s}^{t-1} y_i &- 2 \sum_{i = s+1}^t y_i \geq 1 + \alpha (t - s + 1) &\text{ for } 1 \leq s \leq t, \\
   U(t) \geq 2 y_s &+ \delta \sum_{i = s}^t y_i - 1, \; y_s \geq 0 &\text{ for } 1 \leq s \leq t.
\end{aligned}
\end{equation}
Therefore, by the definition of $U(t)$, the objective value in \eqref{eq:linproblem} must be zero, which completes the proof of the lemma.
\end{proof}

\begin{proof}[Proof of Claim \ref{claim:linproblem}]
Fix any $\lambda = (\Delta^\parallel, \Delta^\perp, \sigma, \gamma)$, which satisfies \eqref{eq:constraint1} and \eqref{eq:constraint2}. Then, let $\lambda' = (\Delta^\parallel, {\Delta^\perp}', \alpha, \gamma)$, where
\begin{equation}
    {\Delta_s^\perp}' := \Delta_s^\perp + (\sigma_s - \alpha) \sum_{i = 1}^s \Delta_i^\parallel \text{ for } 1 \leq s \leq t.
\end{equation}
Note that $\lambda'$ satisfies \eqref{eq:constraint1} and \eqref{eq:constraint2} and the objective values of $\lambda$ and $\lambda'$ are equal. Let $1 \leq l \leq t-1$ be the smallest integer such that $\Delta_l^\parallel < 0$. We will reason inductively on the value of $l$. If such an $l$ does not exist, then let $l = t-1$ by convention and set $\lambda'' = \lambda$. If $l \leq t-1$, then let $\lambda'' = ({\Delta^\parallel}'', {\Delta^\perp}'', \alpha, \gamma'')$, where
\begin{equation}
\begin{aligned}
    {\Delta_s^\parallel}'' &:= \begin{cases}
        0 &\text{ for } s = 1, \\
        \Delta_{s-1}^\parallel &\text{ for } 2 \leq s \leq l, \\
        \Delta_l^\parallel + \Delta_{l+1}^\parallel &\text{ for } s = l+1, \\
        \Delta_s^\parallel &\text{ for } l+2 \leq s \leq t,
    \end{cases} \\
    {\Delta_s^\perp}'' &:= \begin{cases}
        0 &\text{ for } s = 1, \\
        {\Delta_{s-1}^\perp}' &\text{ for } 2 \leq s \leq l, \\
        {\Delta_l^\perp}' + {\Delta_{l+1}^\perp}' + \left\lvert \Delta_l^\parallel \right\rvert + \left\lvert \Delta_{l+1}^\parallel \right\rvert - \left\lvert {\Delta_{l+1}^\parallel}'' \right\rvert + \alpha \sum_{i = 1}^l \Delta_i^\parallel &\text{ for } s = l + 1, \\
        {\Delta_s^\perp}' &\text{ for } l + 2 \leq s \leq t,
    \end{cases} \\
    \gamma_s'' &:= \begin{cases}
        0 &\text{ for } s = 1, \\
        \gamma_{s-1} &\text{ for } 2 \leq s \leq l, \\
        \gamma_l + \gamma_{l+1} &\text{ for } s = l + 1, \\
        \gamma_s &\text{ for } l + 2 \leq s \leq t.
    \end{cases}
\end{aligned}
\end{equation}
Note that $\lambda''$ satisfies \eqref{eq:constraint1} and \eqref{eq:constraint2} and the objective values of $\lambda'$ and $\lambda''$ are equal. Moreover, for any $1 \leq s \leq l$, we know that ${\Delta_s^\parallel}'' \geq 0$. Therefore, we iteratively construct $\lambda''$ according to the procedure stated above until ${\Delta_s^\parallel}'' \geq 0$ for all $1 \leq s \leq t-1$. For the sake of notation, let $\lambda'' = ({\Delta^\parallel}'', {\Delta^\perp}'', \alpha, \gamma'')$ be the solution obtained in this way. If ${\Delta_t^\parallel}'' \leq 0$, then set $\hat{\lambda} = \lambda''$ to complete the proof of the claim. Otherwise, let $\hat{\lambda} = (\hat{\Delta}^\parallel, \hat{\Delta}^\perp, \alpha, \gamma'')$, where
\begin{equation}
\begin{aligned}
    \hat{\Delta}_s^\parallel &:= \begin{cases}
        {\Delta_s^\parallel}'' &\text{ for } 1 \leq s \leq t-1, \\
        0 &\text{ for } s = t,
    \end{cases} \\
    \hat{\Delta}_s^\perp &:= \begin{cases}
        {\Delta_s^\perp}'' &\text{ for } 1 \leq s \leq t-1, \\
        {\Delta_t^\perp}'' + {\Delta_t^\parallel}'' + \alpha {\Delta_t^\parallel}'' &\text{ for } s = t.
    \end{cases}
\end{aligned}
\end{equation}
Note that $\hat{\lambda}$ satisfies \eqref{eq:constraint1} and \eqref{eq:constraint2}, $\hat{\Delta}_t^\parallel \leq 0$ and the objective values of $\lambda''$ and $\hat{\lambda}$ are equal. This concludes the proof of the claim.
\end{proof}

Lemma \ref{lemma:competitive} is the most crucial step in the proof of Proposition \ref{prop:competitive}. The remainder of the proof of Proposition \ref{prop:competitive} largely builds on Lemma \ref{lemma:competitive}. The second statement in Proposition \ref{prop:competitive} follows by identifying one particular feasible solution to the linear optimization problem in \eqref{eq:duallp}.

\begin{proof}[Proof of Proposition \ref{prop:competitive}]
Fix any $k \in \N$. Then,
\begin{equation}
\begin{aligned}
    \sum_{i = M_k+1}^{T_{k+1}} \left( 2 \textsc{Rob}(i) + \textsc{Adv}(i) \right)
    &\geq \sum_{i = M_k+1}^{T_{k+1}} \left( \textsc{Rob}(i) + \lVert r_i - r_{i-1} \rVert + \lVert p_i - p_{i-1} \rVert \right) \\
    &\geq \sum_{i = M_k+1}^{T_{k+1}} \textsc{Rob}(i) + \lVert r_{T_{k+1}} - p_{T_{k+1}} \rVert - \lVert r_{M_k} - p_{M_k} \rVert \\
    &\geq (1 + \delta) \sum_{i = M_k+1}^{T_{k+1}} \textsc{Adv}(i) + \gamma \sum_{i = T_k}^{T_{k+1}} \textsc{Adv}(i),
\end{aligned}
\end{equation}
where the second inequality follows by the triangle inequality and the third inequality follows by line \ref{line:whilerob}. Therefore, since $M_k \geq T_k + 1$,
\begin{equation}
\label{eq:advbounded}
    \sum_{i = T_k}^{T_{k+1}} \textsc{Adv}(i) \leq \frac{2}{\gamma} \sum_{i = T_k+1}^{T_{k+1}} \textsc{Rob}(i).
\end{equation}
Let $K \in \N$ be the largest integer such that $T \geq T_K$. Then,
\begin{equation}
\begin{aligned}
\label{eq:comp1}
    \sum_{i = 1}^{T_K} \textsc{Alg}(i)
    &= \sum_{k = 1}^{K-1} \left( \sum_{i = T_k}^{T_{k+1}-1} \textsc{Alg}(i) + \lVert r_{T_{k+1}-1} - p_{T_{k+1}-1} \rVert - \lVert r_{T_k-1} - p_{T_k-1} \rVert \right) \\
    &+ \textsc{Alg}(T_K) + \lVert p_{T_K} - p_{T_K} \rVert - \lVert r_{T_K-1} - p_{T_K-1} \rVert \\
    &\leq (1 + \delta + \gamma) \sum_{i = 1}^{T_K} \textsc{Adv}(i) \\
    &\leq (1 + \delta + \gamma) \sum_{k = 1}^{K-1} \sum_{i = T_k}^{T_{k+1}} \textsc{Adv}(i) \\
    &\leq \frac{2 + 2 \delta + 2 \gamma}{\gamma} \sum_{i = 1}^{T_K} \textsc{Rob}(i),
\end{aligned}
\end{equation}
where the first inequality follows by Lemma \ref{lemma:consistent} and the third inequality follows by \eqref{eq:advbounded}. Also, if $T \geq M_K$, then
\begin{equation}
\begin{aligned}
\label{eq:comp2a}
    \sum_{i = T_K+1}^T \textsc{Alg}(i)
    &= \sum_{i = T_K+1}^{M_K-1} \textsc{Adv}(i) + f_{M_K}(v_{M_K}) + \lVert p_{M_K-1} - r_{M_K} \rVert + \sum_{i = M_K+1}^T \textsc{Rob}(i) \\
    &\leq \sum_{i = T_K+1}^{M_K-1} \textsc{Adv}(i) + \lVert p_{M_K-1} - r_{M_K-1} \rVert + \sum_{i = M_K}^T \textsc{Rob}(i) \\
    &\leq \sum_{i = T_K+1}^{M_K-1} \textsc{Adv}(i) + \sum_{i = 1}^{M_K-1} \left( \lVert p_i - p_{i-1} \rVert + \lVert r_i - r_{i-1} \rVert \right) + \sum_{i = M_K}^T \textsc{Rob}(i) \\
    &\leq 2 \sum_{i = T_K+1}^{M_K-1} \textsc{Adv}(i) + \sum_{i = 1}^{T_K} \textsc{Adv}(i) + \sum_{i = 1}^T \textsc{Rob}(i) \\
    &\leq 2 U(\infty) \left( \sum_{i = T_K+1}^{M_K-1} \left( \textsc{Rob}(i) + \lVert v_i - r_i \rVert \right) + \lVert p_{T_K} - r_{T_K} \rVert \right) + \sum_{i = 1}^{T_K} \textsc{Adv}(i) + \sum_{i = 1}^T \textsc{Rob}(i) \\
    &\leq 2 U(\infty) \left( \sum_{i = T_K+1}^{M_K-1} \left( \textsc{Rob}(i) + \lVert v_i - r_i \rVert \right) + \sum_{i = 1}^{T_K} \left( \lVert p_i - p_{i-1} \rVert + \lVert r_i - r_{i-1} \rVert \right) \right) \\
    &+ \sum_{i = 1}^{T_K} \textsc{Adv}(i) + \sum_{i = 1}^T \textsc{Rob}(i) \\
    &\leq (2 U(\infty) + 1) \sum_{i = 1}^{T_K} \textsc{Adv}(i) + (2U(\infty) + 1) \sum_{i = 1}^T \textsc{Rob}(i) + 2 U(\infty) \sum_{i = 1}^T \lVert v_i - r_i \rVert \\
    &\leq \left( \frac{4 U(\infty) + 2}{\gamma} + 2 U(\infty) + 1 \right) \sum_{i = 1}^{T} \textsc{Rob}(i) + 2 U(\infty) \sum_{i = 1}^T \lVert v_i - r_i \rVert,
\end{aligned}
\end{equation}
where the first and second inequalities follow by the triangle inequality, the fourth inequality follows by Lemma \ref{lemma:competitive}, the fifth inequality follows again by the triangle inequality and the seventh inequality follows by \eqref{eq:advbounded}. If instead $T \leq M_K - 1$, then
\begin{equation}
\begin{aligned}
\label{eq:comp2b}
    \sum_{i = T_K+1}^T \textsc{Alg}(i)
    &= \sum_{i = T_K+1}^T \textsc{Adv}(i) \\
    &\leq U(\infty) \left( \sum_{i = T_K+1}^T \left( \textsc{Rob}(i) + \lVert v_i - r_i \rVert \right) + \lVert p_{T_K} - r_{T_K} \rVert \right) \\
    &\leq U(\infty) \left( \sum_{i = T_K+1}^T \left( \textsc{Rob}(i) + \lVert v_i - r_i \rVert \right) + \sum_{i = 1}^{T_K} \left( \lVert p_i - p_{i-1} \rVert + \lVert r_i - r_{i-1} \rVert \right) \right) \\
    &\leq U(\infty) \sum_{i = 1}^{T_K} \textsc{Adv}(i) + U(\infty) \sum_{i = 1}^T \left( \textsc{Rob}(i) + \lVert v_i - r_i \rVert \right) + U(\infty) \sum_{i = 1}^T \lVert v_i - r_i \rVert \\
    &\leq \left( \frac{2 U(\infty)}{\gamma} + U(\infty) \right) \sum_{i = 1}^{T} \textsc{Rob}(i) + U(\infty) \sum_{i = 1}^T \lVert v_i - r_i \rVert,
\end{aligned}
\end{equation}
where the first inequality follows by Lemma \ref{lemma:competitive}, the second inequality follows by the triangle inequality and the fourth inequality follows by \eqref{eq:advbounded}. The proof of the first statement is completed by adding \eqref{eq:comp1} and either \eqref{eq:comp2a} or \eqref{eq:comp2b}.

To see the proof of the second statement, assume that $(\alpha, \delta)$ is feasible and $\frac{2}{\alpha \delta} \in \N$. Fix any $t \geq 1$. Note that
\begin{equation}
\begin{aligned}
    U &= \alpha \left( \frac{2}{\alpha + \delta (1 + \alpha)} \right)^{2 / \alpha \delta} + \frac{2}{(2 - \alpha - \delta (1 + \alpha))^2} \left( \alpha \left( \frac{2}{\alpha + \delta (1 + \alpha)} \right)^{2 / \alpha \delta} - \frac{2 - \alpha}{\delta} + 1 \right), \\
    y_{t-s} &= \left( \frac{2}{\alpha + \delta (1 + \alpha)} \right)^s\left( \frac{2 - (s - 1)^+ \alpha \delta}{2 \delta} \right)^+ \text{ for } 0 \leq s \leq t-1,
\end{aligned}
\end{equation}
is feasible for the optimization problem in \eqref{eq:duallp}. Therefore, $U(t) \leq \tilde{U}$ for all $t \geq 1$. This completes the proof of the proposition.
\end{proof}

\section{Proof of Theorem \ref{thm:lowerbound}}
\label{app:proof-lower}

In this section, we will prove Theorem \ref{thm:lowerbound} by proving the next, stronger proposition.

\begin{proposition}
\label{prop:lowerbound}
Fix any $\delta > 0$. Let $\mathcal{A}$ be any deterministic algorithm for the non-convex optimization problem in \eqref{eq:prob-desc}. If there exists $0<\varepsilon < \delta$ such that $\mathcal{A}$ is $(1 + \varepsilon)$-consistent, then $\mathcal{A}$ is at least $\sup_{t \geq 1} L(t)$-robust, where $L(t)$ is as defined in \eqref{eq:lboptproblem}. Moreover, assume that $\frac{2 - \alpha (1 - \delta^2)}{\alpha \delta (1 + \delta)} \in \N$, then
\begin{equation}\label{eq:prop18-2}
    \sup_{t \geq 1} L(t) \geq \frac{\alpha \delta}{4} \left( \frac{2}{\alpha + \delta (1 + \alpha)} \right)^{\frac{2 - \alpha (1 - \delta^2)}{\alpha \delta (1 + \delta)}} - \mathcal{O}(1),
\end{equation}
where the $\mathcal{O}$-notation holds in the limit $\alpha, \delta \to 0$.
\end{proposition}

Note that Theorem \ref{thm:lowerbound} follows directly from Proposition \ref{prop:lowerbound}. Also, note that the assumption that $\frac{2 - \alpha (1 - \delta^2)}{\alpha \delta (1 + \delta)} \in \N$ is again without loss of generality, but prevents rounding symbols from appearing in the notation. The proof of Proposition \ref{prop:lowerbound} depends on the following idea: If $\Delta$ in \eqref{eq:lboptproblem} somehow represents the movement of the advice in one dimension, then the constraints are sufficient to let any algorithm follow the advice exactly, otherwise the algorithm would violate the assumption of $(1 + \delta)$-consistency. Then, we want the advice trajectory which has the maximum cost given these constraints.
This results in the maximization problem in \eqref{eq:lboptproblem}. It turns out that a worst-case solution to \eqref{eq:lboptproblem} and hence a worst-case instance is when the advice moves away exponentially fast from the minimizer.

\begin{proof}[Proof of Proposition \ref{prop:lowerbound}]
Fix any $\delta > 0$, and let $\mathcal{A}$ be any deterministic algorithm. Assume, for the sake of contradiction, that there exists $\varepsilon < \delta$ such that $\mathcal{A}$ is $(1 + \varepsilon)$-consistent and $L$-robust, where $L < \sup_{t \geq 1} L(t)$. 
Then, let $t \geq 1$ be such that $L(t) > L$ and $\Delta$ be an optimal solution to \eqref{eq:lboptproblem}. We will construct an instance in the vector space $X = \R$, where $\lVert x - y \rVert = \lvert x - y \rvert$ and $x_0 = v_0 = p_0 = 0$ . Suppose $v_s = -1$, $p_s = \sum_{i = 1}^s \Delta_i$ and $f_s(x) = \alpha \lvert x - v_s \rvert + \infty \cdot \boldsymbol{1}_{ x \not\in \{ v_s, p_s \}}$ for $1 \leq s \leq t$. Note that $x_s \in \{ v_s, p_s \}$ for all $1 \leq s \leq t$; otherwise, the algorithm would incur an infinite cost, while the algorithm following the minimizer has a finite cost and the proof of the proposition follows trivially. Even more strongly, we claim that $x_s = p_s$ for all $1 \leq s \leq t$. Assuming the claim to be true, observe that
\begin{equation}
\begin{aligned}
    \sum_{i = 1}^t \textsc{Alg}(i)
    &= \sum_{i = 1}^t \textsc{Adv}(i) 
    = \sum_{i = 1}^t \left( \Delta_i + \alpha \left( 1 + \sum_{j = 1}^i \Delta_j \right) \right) \\
    &= L(t)
    = L(t) \sum_{i = 1}^t \textsc{Rob}(i)
    \geq L(t) \sum_{i = 1}^t \textsc{Opt}(i)
    > L \sum_{i = 1}^t \textsc{Opt}(i),
\end{aligned}
\end{equation}
which violates $L$-robustness. This is a contradiction and the proof of the proposition follows. We now provide the proof of the claim that $x_s = p_s$ for all $1 \leq s \leq t$. For the sake of contradiction, let $1 \leq l \leq t$ be the smallest integer such that $x_l = v_l$. As the adversary, we modify the instance for $s \geq l+1$ such that $v_s = p_s = p_l$, $f_s(x) = \infty \cdot \boldsymbol{1}_{ x \not\in \{ p_l \}}$ for $l+1 \leq s \leq t+1$. Note that $x_s = p_l$ for $l+1 \leq s \leq t+1$; otherwise, the algorithm would incur an infinite cost, while the algorithm following the minimizers has a finite cost and the proof of the proposition follows trivially. Then,
\begin{equation}
\begin{aligned}
    \sum_{i = 1}^{t+1} \textsc{Alg}(i)
    &= \sum_{i = 1}^{l-1} \textsc{Adv}(i) + \lvert v_l - p_{l-1} \rvert + \lvert p_l - v_l \rvert
    = \sum_{i = 1}^{l-1} \textsc{Adv}(i) + \left( 1 + \sum_{i = 1}^{l-1} \Delta_i \right) + \left( 1 + \sum_{i = 1}^l \Delta_i \right) \\
    &\geq \sum_{i = 1}^{l-1} \textsc{Adv}(i) + \delta \sum_{i = 1}^l \left( \Delta_i + \alpha \left( 1 + \sum_{j = 1}^i \Delta_j \right) \right) + \Delta_l + \alpha \left( 1 + \sum_{i = 1}^l \Delta_i \right) \\
    &= \sum_{i = 1}^{l-1} \textsc{Adv}(i) + \delta \sum_{i = 1}^l \textsc{Adv}(i) + \textsc{Adv}(l)
    = (1 + \delta) \sum_{i = 1}^{t+1} \textsc{Adv}(i),
\end{aligned}
\end{equation}
where the inequality follows by the first constraint in \eqref{eq:lboptproblem}. This violates the assumption that there exists $\varepsilon < \delta$ such that $\mathcal{A}$ is $(1 + \varepsilon)$-consistent. Now, to see the second statement~\eqref{eq:prop18-2} in the proposition, note that
\begin{equation}
    \Delta_s = \frac{2 - \alpha(1 - \delta^2) - s \alpha \delta (1 + \delta)}{2} \left( \frac{2}{\alpha + \delta (1 + \alpha)} \right)^s \text{ for } 1 \leq s \leq t,
\end{equation}
is feasible for the optimization problem in \eqref{eq:lboptproblem} for $t = \frac{2 - \alpha (1 - \delta^2)}{\alpha \delta (1 + \delta)}$. Also, with this definition,
\begin{equation}
\begin{gathered}
    \sum_{i = 1}^t \left( \Delta_i + \alpha \left( 1 + \sum_{j = 1}^i \Delta_j \right) \right)
    = \frac{\alpha \delta (1 + \delta) (2 + 3 \alpha - \delta (1 + \alpha))}{(2 - \alpha - \delta (1 + \alpha))^3} \left( \frac{2}{\alpha + \delta (1 + \alpha)} \right)^{\frac{2 - \alpha (1 - \delta^2)}{\alpha \delta (1 + \delta)}} \\
    + \frac{2 \alpha^2 (1 + \delta) (2 - 4 \delta + \delta^2 - \alpha (1 + \delta)) + 2 \alpha (2 - \delta) (2 - \delta^2)}{(2 - \alpha - \delta (1 + \alpha))^3}
    - \frac{2 (2 - \delta)^2 (2 + \delta)}{(1 + \delta)(2 - \alpha - \delta (1 + \alpha))^3}.
\end{gathered}
\end{equation}
This completes the proof of the proposition.
\end{proof}

\section{Proof of Proposition \ref{prop:frugal_use} \label{sec:prooffrugal}}

Consider a problem instance $(x_0, f_1, \ldots, f_T, T)$, and let $x_0, x_1, \ldots, x_T$ be the decisions made by algorithm $\mathcal{A}$. Then the cost of the frugal algorithm $\mathcal{A}'$ is
\begin{align}
    &\sum_{t=1}^T f_t(x_{k \cdot \lfloor t/k \rfloor}) + \|x_{k \cdot \lfloor t/k \rfloor} - x_{k \cdot \lfloor t-1/k \rfloor}\| \\
    &\leq \sum_{t=1}^T f_t(x_t) + |f_t(x_t) - f_t(x_{k \cdot \lfloor t/k \rfloor})| + \sum_{i = 1}^{\lfloor T/k \rfloor} \|x_{ki} - x_{k(i-1)}\| \\
    &\leq \sum_{t=1}^T f_t(x_t) + L\|x_t - x_{k \cdot \lfloor t/k \rfloor}\| + \sum_{i = 1}^{\lfloor T/k \rfloor} \|x_{ki} - x_{k(i-1)}\| \label{ineq:lipschitz}\\
    &\leq \sum_{t=1}^T \left(f_t(x_t) + L\sum_{\tau = k\cdot\lfloor t/k \rfloor + 1}^t \|x_\tau - x_{\tau-1}\|\right) + \sum_{i = 1}^{\lfloor T/k \rfloor} \|x_{ki} - x_{k(i-1)}\| \label{ineq:tri_ineq_frugal}\\
    &\leq \max\{1, kl\} \left(\sum_{t=1}^T f_t(x_t) + \|x_t - x_{t-1}\|\right) \label{ineq:frugal_final}
\end{align}
where \eqref{ineq:lipschitz} follows from the Lipschitz bound on the functions $f_t$, \eqref{ineq:tri_ineq_frugal} follows by the triangle inequality, and \eqref{ineq:frugal_final} follows by virtue of the inner sum containing any particular $\|x_\tau - x_{\tau - 1}\|$ at most $k$ times.

\section{Proof of Theorem \ref{thm:memoryless}}
\label{sec:proofmemoryless}

\begin{proof}
We prove Theorem \ref{thm:memoryless} in the vector space $X = \R^2$, where $\lVert x - y \rVert = \lVert x - y \rVert_2$ is the Euclidian distance.
Let $r_1 \leq r_2 < \sqrt{2} / 2$ be two arbitrary numbers to be decided later. Let $x_0 = (0, r_2)$ and $f_1(x, y) = \alpha \lvert x - v_t \rvert + L y$. We let $L \to \infty$ such that any algorithm must move onto the x-axis to achieve a bounded cost (i.e. $x_{1,2} = 0$). Also, let $v_1 = -\frac{r_2^2 - r_1^2}{2 r_1}$ and $\tilde{x}_1 = \sqrt{1 - r_2^2}$. We distinguish two cases.
\begin{enumerate}[(i)]
    \item Assume that $x_{1,1} > r_1$. Then, note that $\lVert x_1 - v_1 \rVert_2 \geq \frac{r_1^2 + r_2^2}{2 r_1} = \lVert x_0 - v_1 \rVert_2$. In other words, the online algorithm moves further away from the minimizer. At time $t = 2$, we rotate the function $f_1$ around $v_1$ and position $\tilde{x}_2$ such that the setup at the next time step is exactly equivalent (up to rotation and scaling) to the setup at time $t = 1$. As $\mathcal{A}$ is memoryless and scale- and rotation invariant, $\mathcal{A}$ will move exactly equivalent as at time $t = 1$ and therefore again move further away from the minimizer. We repeat this setup infinitely often. At each time step the algorithm $\mathcal{A}$ incurs a cost of at least $\sqrt{r_1^2 + r_2^2}$. In contrast, the optimal algorithm moves to $v_1$ at time $t = 1$ and incurs a one-time moving cost of $\frac{r_1^2 + r_2^2}{2 r_1}$. The competitive ratio of $\mathcal{A}$ is therefore unbounded, regardless of the values of $r_1$ and $r_2$.
    \item Assume that $x_{1,1} \leq r_1$. At time $t = 2$, we rotate the function $f_1$ around $\tilde{x}_1$ and position $v_2$ such that the setup at the next time step is exactly equivalent (up to rotation and scaling) to the setup at time $t = 1$. As $\mathcal{A}$ is memoryless and scale- and rotation invariant, $\mathcal{A}$ will move exactly equivalent as at time $t = 1$. We repeat this setup infinitely often. At time $t = 1$, the algorithm incurs a cost of $\sqrt{x_{1,1}^2 + r_2^2} + \alpha \left\lvert x_{1,1} + \frac{r_2^2 - r_1^2}{2 r_1} \right\rvert$. Moreover, the distance to $\tilde{x}_1$ changes from $\lVert x_0 - \tilde{x}_1 \rVert_2 = 1$ to $\lVert x_1 - \tilde{x}_1 \rVert_2 = \sqrt{1 - r_2^2} - x_{1,1} > 0$. Hence, at time $t = 2$ a setup is presented where each distance is scaled by a factor of $\sqrt{1 - r_2^2} - x_{1,1}$. If $\sqrt{1 - r_2^2} - x_{1,1} \neq 1$, then the cost of the algorithm from time $t = 1$ to $T$ is
    \begin{multline}
        \sum_{t = 1}^T \left( \sqrt{1 - r_2^2} - x_{1,1} \right)^{t-1} \left( \sqrt{x_{1,1}^2 + r_2^2} + \alpha \left\lvert x_{1,1} + \frac{r_2^2 - r_1^2}{2 r_1} \right\rvert \right) \\
        = \frac{\left( 1 - \left( \sqrt{1 - r_2^2} - x_{1,1} \right)^T \right) \left( \sqrt{x_{1,1}^2 + r_2^2} + \alpha \left\lvert x_{1,1} + \frac{r_2^2 - r_1^2}{2 r_1} \right\rvert \right)}{1 - \left( \sqrt{1 - r_2^2} - x_{1,1} \right)}.
    \end{multline}
    The cost of the optimal algorithm from time $t = 1$ to $T$ is
    \begin{multline}
        1 + \sum_{t = 1}^T \left( \sqrt{1 - r_2^2} - x_{1,1} \right)^{t-1} \alpha \left( \frac{r_1^2 + r_2^2}{2 r_1} + \sqrt{1 - r_2^2} \right) \\
        = 1 + \frac{\left( 1 - \left( \sqrt{1 - r_2^2} - x_{1,1} \right)^T \right) \alpha \left( \frac{r_1^2 + r_2^2}{2 r_1} + \sqrt{1 - r_2^2} \right)}{1 - \left( \sqrt{1 - r_2^2} - x_{1,1} \right)}.
    \end{multline}
    Let $r_1 = \alpha$ and $r_2 = \sqrt{2 \alpha}$. We distinguish two more cases. If $\sqrt{1 - r_2^2} - x_{1,1} \geq 1$, then, as $T \to \infty$, the competitive ratio of $\mathcal{A}$ is at least
    \begin{equation}
        \textsc{CR} \geq \frac{\sqrt{x_{1,1}^2 + r_2^2} + \alpha \left\lvert x_{1,1} + \frac{r_2^2 - r_1^2}{2 r_1} \right\rvert}{\alpha \left( \frac{r_1^2 + r_2^2}{2 r_1} + \sqrt{1 - r_2^2} \right)}
        \geq \frac{\sqrt{2 - 2 \sqrt{1 - 2 \alpha}}}{\alpha \left( \frac{\alpha^2 + 2 \alpha}{2 \alpha} + \sqrt{1 - 2 \alpha} \right)}
        = \frac{1}{\sqrt{2 \alpha}} - o\left( \frac{1}{\sqrt{\alpha}} \right).
    \end{equation}
    If $\sqrt{1 - r_2^2} - x_{1,1} < 1$, then, as $T \to \infty$, the competitive ratio of $\mathcal{A}$ is at least
    \begin{multline}
        \textsc{CR} \geq \frac{\sqrt{x_{1,1}^2 + r_2^2} + \alpha \left( x_{1,1} + \frac{r_2^2 - r_1^2}{2 r_1} \right)}{1 - \left( \sqrt{1 - r_2^2} - x_{1,1} \right) + \alpha \left( \frac{r_1^2 + r_2^2}{2 r_1} + \sqrt{1 - r_2^2} \right)} \\
        \geq \frac{ \sqrt{\alpha^2 + 2 \alpha} + \alpha \left( \alpha + \frac{2 \alpha - \alpha^2}{2 \alpha} \right)}{1 - \left( \sqrt{1 - 2 \alpha} - \alpha \right) + \alpha \left( \frac{\alpha^2 + 2 \alpha}{2 \alpha} + \sqrt{1 - 2 \alpha} \right)}
        = \frac{1}{\sqrt{8 \alpha}} - o\left( \frac{1}{\sqrt{\alpha}} \right).
    \end{multline}
\end{enumerate}
\end{proof}

\section{Proof of Theorem \ref{thm:onedim}}
\label{sec:proofonedim}

\begin{proof}[\unskip\nopunct]
Let $y \in \R^T$ be an arbitrary solution and define the potential function $\phi(y_t, x_t): = c \lvert y_t - x_t \rvert$ for $c > 0$. Note that if we prove that
\begin{equation}
    f_t(x_t) + \lvert x_t - x_{t-1} \rvert + \phi(y_t, x_t) - \phi(y_{t-1}, x_{t-1})
    \leq \textsc{CR} \left( f_t(y_t) + \lvert y_t - y_{t-1} \rvert \right),
\end{equation}
for all $1 \leq t \leq T$, then, if we sum over $t$, we obtain
\begin{equation}
\begin{aligned}
    \sum_{i = 1}^T \left( f_t(x_t) + \lvert x_t - x_{t-1} \rvert \right)
    &\leq \textsc{CR} \sum_{i = 1}^T \left( f_t(y_t) + \lvert y_t - y_{t-1} \rvert \right) - \phi(y_T, x_T) \\
    &\leq \textsc{CR} \sum_{i = 1}^T \left( f_t(y_t) + \lvert y_t - y_{t-1} \rvert \right),
\end{aligned}
\end{equation}
which proves that $x$ is \textsc{CR}-competitive \emph{with respect to} $y$. We will apply this technique to $y = o$ and $y = p$ separately to find the competitive ratio with respect to the hindsight optimal and the advice, respectively. Applying the triangle equality to $\phi(y_t, x_t) - \phi(y_{t-1}, x_{t-1})$ yields
\begin{multline}
    \phi(y_t, x_t) - \phi(y_{t-1}, x_{t-1})
    = c \left( \lvert y_t - x_t \rvert - \lvert y_{t-1} - x_{t-1} \rvert \right) \\
    \leq c \left( \lvert y_t - y_{t-1} \rvert + \lvert y_t - x_t \rvert - \lvert y_t - x_{t-1} \rvert \right) \\
    \leq \textsc{CR} \cdot \lvert y_t - y_{t-1} \rvert + c \left( \lvert y_t - x_t \rvert - \lvert y_t - x_{t-1} \rvert \right),
\end{multline}
where we assume that $\textsc{CR} \geq c$. This means it is sufficient to prove that
\begin{equation}
\label{eq:suffcondcr}
    f_t(x_t) + \lvert x_t - x_{t-1} \rvert + c \left( \lvert y_t - x_t \rvert - \lvert y_t - x_{t-1} \rvert \right)
    \leq \textsc{CR} \cdot f_t(y_t).
\end{equation}
We verify equation \eqref{eq:suffcondcr} in the case that $y = o$ first. Let $c = 1 + \underline{\beta}^{-1}$ and $\textsc{CR} = 1 + (2 + \underline{\beta}^{-1}) \bar{\beta}$ as in the theorem. We distinguish two cases.
\begin{enumerate}
    \item
Assume that $f_t(x_t) \leq f_t(o_t)$. Note that in any case
\begin{equation}
    \lvert x_t - x_{t-1} \rvert
    \leq \lvert x(\bar{\lambda}) - x_{t-1} \rvert
    \leq \bar{\beta} f_t(x(\bar{\lambda}))
    \leq \bar{\beta} f_t(x_t),
\end{equation}
where the last inequality follows by the convexity of $f_t$ and hence, by applying the triangle inequality,
\begin{equation}
\begin{aligned}
    f_t(x_t) + \lvert x_t - x_{t-1} \rvert &+ c \left( \lvert o_t - x_t \rvert - \lvert o_t - x_{t-1} \rvert \right) \\
    &\leq f_t(x_t) + (1 + c) \lvert x_t - x_{t-1} \rvert \\
    &\leq \left( 1 + (1 + c) \bar{\beta} \right) f_t(x_t)
    \leq \textsc{CR} \cdot f_t(o_t),
\end{aligned}
\end{equation}
which verifies equation \eqref{eq:suffcondcr}.

    \item
Assume that $f_t(x_t) > f_t(o_o)$. Note that in this case $x_t$ did not reach $v_t$ which means that $\underline{\lambda} < 1$ and $\lvert x(\underline{\lambda}) - x_{t-1} \rvert = \underline{\beta} f_t(x(\underline{\lambda}))$. Thus,
\begin{equation}
    f_t(x_t)
    \leq f_t(x(\underline{\lambda}))
    = \frac{\lvert x(\underline{\lambda}) - x_{t-1}\rvert}{\underline{\beta}}
    \leq \frac{\lvert x_t - x_{t-1}\rvert}{\underline{\beta}}.
\end{equation}
Moreover, since $f_t(x_t) > f_t(o_t)$, $x_t$ must have moved closer to $o_t$ during its entire move and thus
\begin{equation}
    c \left( \lvert o_t - x_t \rvert - \lvert o_t - x_{t-1} \rvert \right)
    = -c \lvert x_t - x_{t-1} \rvert.
\end{equation}
Therefore,
\begin{equation}
    f_t(x_t) + \lvert x_t - x_{t-1} \rvert + c \left( \lvert o_t - x_t \rvert - \lvert o_t - x_{t-1} \rvert \right)
    \leq \left( 1 + \underline{\beta}^{-1} - c \right) \lvert x_t - x_{t-1} \rvert
    \leq 0,
\end{equation}
which verifies equation \eqref{eq:suffcondcr}.
\end{enumerate}
We now continue to verify equation \eqref{eq:suffcondcr} in the case that $y = p$. Let $c = 1 + \bar{\beta}^{-1}$ and $\textsc{CR} = 1 + (2 + \bar{\beta}^{-1}) \underline{\beta}$ as in the theorem. We distinguish three cases.
\begin{enumerate}
    \item
Assume that $f_t(x_t) \leq f_t(p_t)$ and $x_t \neq p_t$. Either $p_t$ is on the opposite side of $v_t$ as $x_t$ or $p_t$ is on the same side of $v_t$ as $x_t$. If $p_t$ is on the same side, then it must be that $x_t = x(\underline{\lambda})$, since by decreasing $\lambda$, the point $x(\lambda)$ only moves closer to $p_t$. Then, by applying the triangle inequality,
\begin{equation}
\begin{aligned}
    f_t(x_t) + \lvert x_t - x_{t-1} \rvert + c \left( \lvert p_t - x_t \rvert - \lvert p_t - x_{t-1} \rvert \right)
    &\leq f_t(x_t) + (1 + c) \lvert x_t - x_{t-1} \rvert \\
    &\leq \left( 1 + (1 + c) \underline{\beta} \right) f_t(x_t)
    \leq \textsc{CR}\cdot  f_t(p_t),
\end{aligned}
\end{equation}
which verifies equation \eqref{eq:suffcondcr}. If $p_t$ is on the opposite side of $v_t$, then $x_t$ must have moved closer to $p_t$ during its entire move and thus
\begin{equation}
    c \left( \lvert p_t - x_t \rvert - \lvert p_t - x_{t-1} \rvert \right)
    = -c \lvert x_t - x_{t-1} \rvert.
\end{equation}
Therefore,
\begin{equation}
    f_t(x_t) + \lvert x_t - x_{t-1} \rvert + c \left( \lvert p_t - x_t \rvert - \lvert p_t - x_{t-1} \rvert \right)
    \leq f_t(x_t)
    \leq \textsc{CR} \cdot f_t(p_t),
\end{equation}
which verifies equation \eqref{eq:suffcondcr}.

    \item
Assume that $x_t = p_t$. Then,
\begin{equation}
    f_t(x_t) + \lvert x_t - x_{t-1} \rvert + c \left( \lvert p_t - x_t \rvert - \lvert p_t - x_{t-1} \rvert \right)
    \leq f_t(p_t)
    \leq \textsc{CR} \cdot f_t(p_t)
\end{equation}
which verifies equation \eqref{eq:suffcondcr}.

    \item
Assume that $f_t(x_t) > f_t(p_t)$. Then it must be that $x_t = x(\bar{\lambda})$, since by increasing $\lambda$, the point $x(\lambda)$ only moves closer to $p_t$. Moreover, $x_t$ did not reach $v_t$ which means that $\bar{\lambda} < 1$ and thus
\begin{equation}
    f_t(x_t) = f_t(x(\bar{\lambda})) = \frac{\lvert x(\bar{\lambda}) - x_{t-1} \rvert}{\bar{\beta}} = \frac{\lvert x_t - x_{t-1} \rvert}{\bar{\beta}}. 
\end{equation}
Also, since $f_t(x_t) > f_t(p_t)$, $x_t$ must have moved closer to $o_t$ during its entire move and thus
\begin{equation}
    c \left( \lvert p_t - x_t \rvert - \lvert p_t - x_{t-1} \rvert \right)
    = -c \lvert x_t - x_{t-1} \rvert.
\end{equation}
Therefore,
\begin{equation}
    f_t(x_t) + \lvert x_t - x_{t-1} \rvert + c \left( \lvert p_t - x_t \rvert - \lvert p_t - x_{t-1} \rvert \right)
    = \left( 1 + \bar{\beta}^{-1} - c \right) \lvert x_t - x_{t-1} \rvert
    \leq 0,
\end{equation}
which verifies equation \eqref{eq:suffcondcr}.
\end{enumerate}
This completes the proof of the theorem by applying Lemma \ref{lemma:ftpalg}.
\end{proof}

\section{Proof of Theorem \ref{thm:lbonedim}}
\label{sec:prooflbonedim}

\begin{proof}[\unskip\nopunct]
Let $\mathcal{A}$ be any deterministic algorithm for the convex, one-dimensional optimization problem and fix any $0 < \delta < 1/2$. Let $x$ denote the decisions of $\mathcal{A}$. We construct an instance. Let $\tilde{x}_0 = x_0 = 0$, $f_1(x) = 2 \delta \lvert x - 1 \rvert$, the advice $\tilde{x}_1 = 1$ and $T = 2$. We distinguish two cases.
\begin{enumerate}
    \item
Assume that $x_1 \geq \frac{1}{2}$. Let $f_2(x) = \lvert x \rvert$ and the advice $\tilde{x}_2 = 0$. Then, the optimal solution incurs a cost of $2 \delta$ with $o_1 = o_2 = 0$. The algorithm $\mathcal{A}$ has a cost of at least $1$ and hence the competitive ratio is at least $1 / (2 \delta)$.

    \item
Assume that $x_1 < \frac{1}{2}$. Let $f_2(x) = \lvert x - 1 \rvert$ and the advice $\tilde{x}_2 = 1$. Then, the optimal solution incurs a cost of $1$ with $o_1 = o_2 = 1$. However, the algorithm $\mathcal{A}$ has a cost of at least $1 + 2 \delta (1 - x_1) > 1 + \delta$ even though the predictions are perfect. Hence, this case does not satisfy the assumption that $\mathcal{A}$ is $(1 + \delta)$-consistent.
\end{enumerate}
\end{proof}

\end{document}